\documentclass[nohyperref]{article}

\usepackage{microtype}
\usepackage{graphicx}
\usepackage{subfigure}
\usepackage{booktabs} 

\usepackage{hyperref}

\usepackage{paralist, tabularx}

 \usepackage[accepted]{icml2023}

\usepackage{amsmath}
\usepackage{amssymb}
\usepackage{mathtools}
\usepackage{amsthm}

\usepackage[capitalize,noabbrev]{cleveref}

\theoremstyle{plain}

\usepackage{amsthm} 
\newtheorem{innercustomAS}{AS}
\newenvironment{customAS}[1]
  {\renewcommand\theinnercustomAS{#1}\innercustomAS}
  {\endinnercustomAS}

\newtheorem{theorem}{Theorem}[section]

\newtheorem{lemma}[theorem]{Lemma}

\theoremstyle{definition}
\newtheorem{definition}[theorem]{Definition}
\theoremstyle{remark}
\newtheorem{remark}[theorem]{Remark}

\usepackage[textsize=tiny]{todonotes}

\usepackage[utf8]{inputenc}
\usepackage{xspace}
\usepackage{caption}
\usepackage{mathtools}
\usepackage{thm-restate}
\usepackage{booktabs}       
\usepackage{amsfonts}
\usepackage{amsmath}
\usepackage{lipsum}
\usepackage{amsthm}
\usepackage{amssymb}
\usepackage{tikz-cd} 
\usepackage{makecell}
\usepackage{thm-restate}
\usepackage{thmtools}
\usepackage{lipsum}
\usepackage{thm-restate}
\usepackage{hyperref}
\usepackage[utf8]{inputenc} 
\usepackage[T1]{fontenc}    
\usepackage{hyperref}       
\usepackage{url}            
\usepackage{booktabs}       
\usepackage{amsfonts}       
\usepackage{amsmath}
\usepackage{nicefrac}       
\usepackage{microtype}      
\usepackage{enumitem}
\usepackage{subfigure}
\usepackage{wrapfig}
\usepackage[capitalize]{cleveref}
\usepackage{adjustbox}
\usepackage{makecell, multirow}
\usepackage{stackrel} 
\newcommand{\tabtopvspace}{{\vspace{-0pt}}}

\input{./mysymbol.sty}

\usepackage{amsmath,amsfonts,bm}

\def\eqref#1{equation~\ref{#1}}

\def\1{\bm{1}}

\def\va{{\bm{a}}}
\def\vb{{\bm{b}}}
\def\vc{{\bm{c}}}

\def\ve{{\bm{e}}}

\def\vk{{\bm{k}}}

\def\vq{{\bm{q}}}

\def\vu{{\bm{u}}}
\def\vv{{\bm{v}}}
\def\vw{{\bm{w}}}
\def\vx{{\bm{x}}}
\def\vy{{\bm{y}}}
\def\vz{{\bm{z}}}

\def\mA{{\bm{A}}}
\def\mB{{\bm{B}}}

\def\mF{{\bm{F}}}

\def\mH{{\bm{H}}}

\def\mL{{\bm{L}}}

\def\mV{{\bm{V}}}
\def\mW{{\bm{W}}}
\def\mX{{\bm{X}}}

\DeclareMathAlphabet{\mathsfit}{\encodingdefault}{\sfdefault}{m}{sl}
\SetMathAlphabet{\mathsfit}{bold}{\encodingdefault}{\sfdefault}{bx}{n}

\definecolor{dark2green}{rgb}{0.1, 0.65, 0.3}
\definecolor{dark2orange}{rgb}{0.9, 0.4, 0.}
\definecolor{dark2purple}{rgb}{0.4, 0.4, 0.8}

\usepackage[utf8]{inputenc} 
\usepackage{hyperref}       
\usepackage{booktabs}       
\usepackage{nicefrac}       
\usepackage{microtype}      
\usepackage{multirow}
\usepackage{changepage,threeparttable} 
\usepackage{makecell}
\usepackage{colortbl}
\usepackage{xspace}
\usepackage{xcolor}  

\setcellgapes{2.5pt}

\icmltitlerunning{On the Connection Between MPNN and Graph Transformer}

\begin{document}

\twocolumn[
\icmltitle{On the Connection Between MPNN and Graph Transformer}



\icmlsetsymbol{equal}{*}
\begin{icmlauthorlist}
\icmlauthor{Chen Cai}{ucsd}
\icmlauthor{Truong Son Hy}{ucsd}
\icmlauthor{Rose Yu}{ucsd}
\icmlauthor{Yusu Wang}{ucsd}
\end{icmlauthorlist}
\icmlaffiliation{ucsd}{University of California San Diego, San Diego, USA}

\icmlcorrespondingauthor{Chen Cai}{c1cai@ucsd.edu}

\icmlkeywords{graph neural networks}

\vskip 0.3in
]



\printAffiliationsAndNotice{} 

\begin{abstract}
Graph Transformer (GT) recently has emerged as a new paradigm of graph learning algorithms, outperforming the previously popular Message Passing Neural Network (MPNN) on multiple benchmarks. Previous work \citep{kim2022pure} shows that with proper position embedding, GT can approximate MPNN arbitrarily well, implying that GT is at least as powerful as MPNN. In this paper, we study the inverse connection and show that MPNN with virtual node (VN), a commonly used heuristic with little theoretical understanding, is powerful enough to arbitrarily approximate the self-attention layer of GT.  
In particular, we first show that if we consider one type of linear transformer, the so-called Performer/Linear Transformer \citep{choromanski2020rethinking,katharopoulos2020transformers}, then MPNN + VN with only $\mathcal{O}(1)$ depth and $\mathcal{O}(1)$ width can approximate a self-attention layer in Performer/Linear Transformer. 
Next, via a connection between MPNN + VN and DeepSets, we prove the MPNN + VN with $\mathcal{O}(n^d)$ width and $\mathcal{O}(1)$ depth can approximate the self-attention layer arbitrarily well, where $d$ is the input feature dimension. Lastly, under some assumptions, we provide an explicit construction of MPNN + VN with $\mathcal{O}(1)$ width and $\mathcal{O}(n)$ depth approximating the self-attention layer in GT arbitrarily well.
On the empirical side, we demonstrate that 1) MPNN + VN is a surprisingly strong baseline, outperforming GT on the recently proposed Long Range Graph Benchmark (LRGB) dataset, 2) our MPNN + VN improves over early implementation on a wide range of OGB datasets and 3) MPNN + VN outperforms Linear Transformer and MPNN on the climate modeling task. 

\end{abstract}

\section{Introduction}
MPNN (Message Passing Neural Network) \cite{gilmer2017neural} has been the leading architecture for processing graph-structured data. Recently, transformers in natural language processing \citep{vaswani2017attention,kalyan2021ammus} and vision \citep{d2021convit,han2022survey} have extended their success to the domain of graphs. There have been several pieces of work \citep{,ying2021transformers,wu2021representing,kreuzer2021rethinking,rampavsek2022recipe, kim2022pure} showing that with careful position embedding \citep{lim2022sign}, graph transformers (GT) can achieve compelling empirical performances on large-scale datasets and start to challenge the dominance of MPNN. 
\begin{figure}[t!]
  \centering
  \includegraphics[width=1\linewidth]{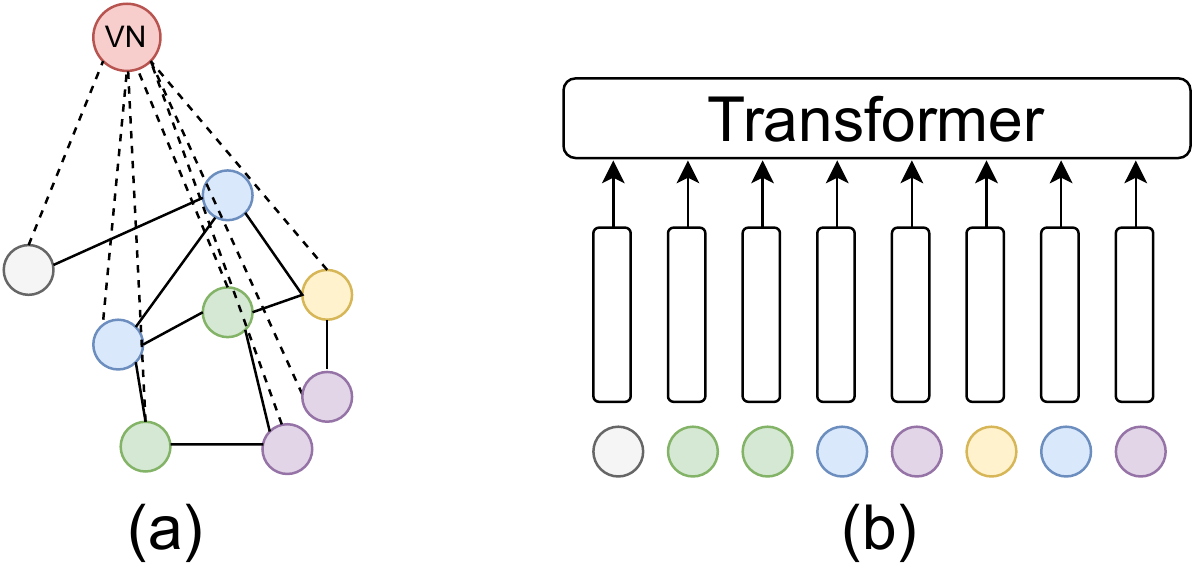}
\caption{MPNN + VN and Graph Transformers.}
\label{fig:mpnn+gt}
\end{figure}

\begin{table*}[th!]
\centering
\caption{Summary of approximation result of MPNN + VN on self-attention layer. $n$ is the number of nodes and $d$ is the feature dimension of node features. The dependency on $d$ is hidden. }
\tabtopvspace
\label{tab:theoretical-result}
\scalebox{1}{
\begin{tabular}{@{}lllll@{}}
\toprule
 & Depth & Width & Self-Attention & Note \\ \midrule
 \Cref{thm:constant-depth-constant-width} & $\mc{O}(1)$ & $\mc{O}(1)$ & Approximate & Approximate self attention in Performer \cite{choromanski2020rethinking} \\ 
 \Cref{thm:constant-depth} & $\mc{O}(1)$ & $\mc{O}(n^d)$ & Full & Leverage the universality of equivariant \DS{} \\
\Cref{thm:constant-width}  & $\mc{O}(n)$ & $\mc{O}(1)$ & Full & Explicit construction, strong assumption on $\mc{X}$ \\
\Cref{prop:gat-v2-selection} & $\mc{O}(n)$ & $\mc{O}(1)$ & Full & Explicit construction, more relaxed (but still strong) assumption on $\mc{X}$ \\ \bottomrule
\end{tabular}
}
\end{table*}

MPNN imposes a sparsity pattern on the computation graph and therefore enjoys linear complexity. It however suffers from well-known over-smoothing \citep{li2018deeper,oono2019graph,cai2020note} and over-squashing \citep{alon2020bottleneck,topping2021understanding} issues, limiting its usage on long-range modeling tasks where the label of one node depends on features of nodes far away. GT relies purely on position embedding to encode the graph structure and uses vanilla transformers on top. \footnote{GT in this paper refers to the practice of tokenizing graph nodes and applying standard transformers on top \citep{ying2021transformers,kim2022pure}. There exists a more sophisticated GT \citep{kreuzer2021rethinking} that further conditions attention on edge types but it is not considered in this paper. } 
It models all pairwise interactions directly in one layer, making it computationally more expensive. Compared to MPNN, GT shows promising results on tasks where modeling long-range interaction is the key, but the quadratic complexity of self-attention in GT limits its usage to graphs of medium size. Scaling up GT to large graphs remains an active research area \citep{wunodeformer}. 

Theoretically, it has been shown that graph transformers can be powerful graph learners \citep{kim2022pure}, i.e., graph transformers with appropriate choice of token embeddings have the capacity of approximating linear permutation equivariant basis, and therefore can approximate 2-IGN (Invariant Graph Network), a powerful architecture that is at least as expressive as MPNN \citep{maron2018invariant}. This raises an important question that \textit{whether GT is strictly more powerful than MPNN}. Can we approximate GT with MPNN?

One common intuition of the advantage of GT over MPNN is its ability to model long-range interaction more effectively. However, from the MPNN side, one can resort to a simple trick to escape locality constraints for effective long-range modeling: the use of an additional \emph{virtual node (VN)} that connects to all input graph nodes. On a high level, MPNN + VN augments the existing graph with one virtual node, which acts like global memory for every node exchanging messages with other nodes. Empirically this simple trick has been observed to improve the MPNN and has been widely adopted \citep{gilmer2017neural,hu2020open,hu2021ogb} since the early beginning of MPNN \citep{gilmer2017neural, battaglia2018relational}. However, there is very little theoretical study of MPNN + VN \citep{hwanganalysis}. 

In this work, we study the theoretical property of MPNN + VN, and its connection to GT. We systematically study the representation power of MPNN + VN, both for certain approximate self-attention and for the full self-attention layer, and provide a depth-width trade-off, summarized in \Cref{tab:theoretical-result}. In particular, 
\begin{itemize}
\item With $\mc{O}(1)$ depth and $\mc{O}(1)$ width, MPNN + VN can approximate one self-attention layer of Performer \citep{choromanski2020rethinking} and Linear Transformer \citep{katharopoulos2020transformers}, a type of linear transformers \citep{tay2020efficient}. 
\item 
Via a link between MPNN + VN with \DS{} \citep{zaheer2017deep}, we prove MPNN + VN with $\mc{O}(1)$ depth and $\mc{O}(n^d)$ width ($d$ is the input feature dimension) is permutation equivariant universal, implying it can approximate self-attention layer and even full-transformers.  

\item 
Under certain assumptions on node features, we prove an explicit construction of $\mc{O}(n)$ depth $\mc{O}(1)$ width MPNN + VN approximating 1 self-attention layer arbitrarily well on graphs of size $n$. 
Unfortunately, the assumptions on node features are rather strong, and whether we can alleviate them will be an interesting future direction to explore. 

\item Empirically, we show 1) that MPNN + VN works surprisingly well on the recently proposed LRGB (long-range graph benchmarks) datasets \citep{dwivedi2022long}, which arguably require long-range interaction reasoning to achieve strong performance 2) our implementation of MPNN + VN is able to further improve the early implementation of MPNN + VN on OGB datasets and 3) MPNN + VN outperforms Linear Transformer \citep{katharopoulos2020transformers} and MPNN on the climate modeling task. 
\end{itemize}

\section{Related Work}
\textbf{Virtual node in MPNN.}
The virtual node augments the graph with an additional node to facilitate the information exchange among all pairs of nodes. It is a heuristic proposed in \cite{gilmer2017neural} and has been observed to improve the performance in different tasks \citep{hu2021ogb,hu2020open}. Surprisingly, its theoretical properties have received little study. To the best of our knowledge, only a recent paper \citep{hwanganalysis} analyzed the role of the virtual node in the link prediction setting in terms of 1) expressiveness
of the learned link representation and 2) the potential impact on under-reaching and over-smoothing. 

\textbf{Graph transformer.}
Because of the great successes of Transformers in natural language processing (NLP) \citep{vaswani2017attention,wolf2020transformers} and recently in computer vision \citep{dosovitskiy2020image,d2021convit,liu2021swin}, there is great interest in extending transformers for graphs \citep{muller2023attending}. One common belief of advantage of graph transformer over MPNN is its capacity in capturing long-range interactions while alleviating over-smoothing \citep{li2018deeper,oono2019graph,cai2020note} and over-squashing in MPNN \citep{alon2020bottleneck,topping2021understanding}. 

Fully-connected Graph transformer \citep{dwivedi2020generalization} was introduced with eigenvectors of graph Laplacian as the node positional encoding (PE). Various follow-up works proposed different ways of PE to improve GT, ranging from an invariant aggregation of Laplacian's eigenvectors in SAN \citep{kreuzer2021rethinking}, pair-wise graph distances in Graphormer \citep{ying2021transformers}, relative PE derived from diffusion kernels in GraphiT \citep{mialon2021graphit}, and recently Sign and Basis Net \citep{lim2022sign} with a principled way of handling sign and basis invariance. 
Other lines of research in GT include combining MPNN and GT  \citep{wu2021representing,rampavsek2022recipe}, encoding the substructures \citep{chen2022structure}, GT for directed graphs \citep{geisler2023transformers}, and efficient graph transformers for large graphs \citep{wunodeformer}.

\textbf{Deep Learning on Sets.} Janossy pooling \citep{murphy2018janossy} is a framework to build permutation invariant architecture for sets using permuting \& averaging paradigm while limiting the number of elements in permutations to be $k < n$.  Under this framework, \DS{} \citep{zaheer2017deep} and PointNet \citep{qi2017pointnet} are recovered as the case of $k=1$. For case $k=2$, self-attention and Relation Network \citep{santoro2017simple} are recovered \citep{wagstaff2022universal}. Although \DS{} and Relation Network \citep{santoro2017simple} are both shown to be universal permutation invariant, recent work \citep{zweig2022exponential} provides a finer characterization on the representation gap between the two architectures.  


\section{Preliminaries}
We denote $\mX\in \R{n \times d}$ the concatenation of graph node features and positional encodings, where node $i$ has feature $\vx_i \in \R{d}$. When necessary, we use $\vx^{(l)}_j$ to denote the node $j$'s feature at depth $l$.  Let $\mc{M}$ be the space of multisets of vectors in $\R{d}$. We use $\mc{X} \subseteq \R{n\times d} $ to denote the space of node features and the $\mc{X}_i$ be the projection of $\mc{X}$ on $i$-th coordinate. $\norm{\cdot}$ denotes the 2-norm. $[\vx, \vy, \vz]$ denotes the concatenation of $\vx, \vy, \vz$. $[n]$ stands for the set $\{1, 2, ..., n\}$.

\begin{definition}[attention]
We denote key and query matrix as $\mW_K, \mW_Q\in \R{d \times d'}$, and value matrix as $\mW_V \in \R{d \times d}$ \footnote{For simplicity, we assume the output dimension of self-attention is the same as the input dimension. All theoretical results can be extended to the case where the output dimension is different from $d$.}. Attention score between two vectors $\vu, \vv \in \R{d \times 1}$ is defined as $\alpha(\vu, \vv) = \text{softmax}(\vu^T \mW_Q(\mW_K)^T\vv)$. We denote $\mc{A}$ as the space of attention $\alpha$ for different $\mW_Q, \mW_K, \mW_V$. We also define unnormalized attention score $\alpha'(\cdot, \cdot)$ to be $\alpha'(\vu, \vv) = \vu^T \mW_Q(\mW_K)^T\vv$.
Self attention layer is a matrix function $\mL: \R{n\times d} \rightarrow \R{n\times d}$  of the following form: $\mL(\mX) = \text{softmax}(\mX\mW_Q(\mX\mW_K)^T)\mX \mW_V$. 
\end{definition}

\subsection{MPNN Layer}
\begin{definition}[MPNN layer \citep{gilmer2017neural}]
An MPNN layer on a graph $G$ with node features $\vx^{(k)}$ at $k$-th layer and edge features $\ve$ is of the following form
\begin{equation*}
\vx_i^{(k)}=\gamma^{(k)}\left(\vx_i^{(k-1)}, \pool_{j \in \mc{N}(i)} \phi^{(k)}\left(\vx_i^{(k-1)}, \vx_j^{(k-1)}, \ve_{j, i}\right)\right)
\end{equation*}

Here $\gamma: \mb{R}^d \times \mb{R}^{d'} \rightarrow \mb{R}^d$ is update function,
$\phi:\mb{R}^d \times \mb{R}^{d} \times \mb{R}^{d_e} \rightarrow \mb{R}^{d'}$ is message function where $d_e$ is the edge feature dimension, 
$\pool: \mc{M}  \rightarrow \mb{R}^d$ is permutation invariant aggregation function 
and $\mathcal{N}(i)$ is the neighbors of node $i$ in $G$.
Update/message/aggregation functions are usually parametrized by neural networks. For graphs of different types of edges and nodes, one can further extend MPNN to the heterogeneous setting.  We use $1, ..., n$ to index graph nodes and $\vn$ to denote the virtual node. 
\end{definition}



\begin{definition}[heterogeneous MPNN + VN layer]\label{def-hetero-mpnn-vn-layer} 
The heterogeneous MPNN + VN layer operates on two types of nodes: 1) virtual node and 2) graph nodes, denoted as \text{vn} and \text{gn}, and three types of edges: 1) \text{vn}-\text{gn} edge and 2) \text{gn}-\text{gn} edges and 3) \text{gn}-\text{vn} edges. It has the following form

\begin{equation}
\vx_{\vn}^{(k)}=\gamma_{\vn}^{(k)}\left(\vx_i^{(k-1)}, \pool_{j \in [n]}  \phi^{(k)}_{\vngn}\left(\vx_i^{(k-1)}, \vx_j^{(k-1)}, \ve_{j, i}\right)\right)  
\end{equation}
for the virtual node, and 
\begin{equation}
\begin{split}
\vx_i^{(k)}&=\gamma_{\gn}^{(k)}(\vx_i^{(k-1)}, \pool_{j \in \mathcal{N}_1(i)}  \phi^{(k)}_{\gnvn}\left(\vx_i^{(k-1)}, \vx_j^{(k-1)}, \ve_{j, i}\right)  \\
& +  \pool_{j \in \mathcal{N}_2(i)} \phi^{(k)}_{\gngn}\left(\vx_i^{(k-1)}, \vx_j^{(k-1)}, \ve_{j, i})\right)  
\end{split}
\end{equation}
for graph node. Here $\mc{N}_1(i)$ for graph node $i$ is the virtual node and $\mc{N}_2(i)$ is the set of neighboring graph nodes.
\end{definition}
Our proof of approximating self-attention layer $\mL$ with MPNN layers does not use the graph topology. Next, we introduce a simplified heterogeneous MPNN + VN layer, which will be used in the proof. It is easy to see that setting $\phi^{(k)}_{\gngn}$ to be 0 in \Cref{def-hetero-mpnn-vn-layer} recovers the simplified heterogeneous MPNN + VN layer.

\begin{definition}[simplified heterogeneous MPNN + VN layer]
\label{def:simplified-hetero-mpnn-vn}
A simplified heterogeneous MPNN + VN layer is the same as a heterogeneous MPNN + VN layer in \Cref{def-hetero-mpnn-vn-layer} except we set $\theta_{\gngn}$ to be 0. I.e., we have
\begin{equation*}
\vx_{\vn}^{(k)}=\gamma_{\vn}^{(k)}\left(\vx_i^{(k-1)}, \pool_{j \in [n]}  \phi^{(k)}_{\vngn}\left(\vx_i^{(k-1)}, \vx_j^{(k-1)}, \ve_{j, i}\right)\right)  
\end{equation*}
for the virtual node, and
\begin{equation*}
\vx_i^{(k)}=\gamma_{\gn}^{(k)}\left(\vx_i^{(k-1)}, \pool_{j \in \mathcal{N}_1(i)}  \phi^{(k)}_{\gnvn}\left(\vx_i^{(k-1)}, \vx_j^{(k-1)}, \ve_{j, i}\right)\right)  
\end{equation*}
for graph nodes. 
\end{definition}

Intuitively, adding the virtual node (VN) to MPNN makes it easy to compute certain quantities, for example, the mean of node features (which is hard for standard MPNN unless the depth is proportional to the diameter of the graph). Using VN thus makes it easy to implement for example the mean subtraction, which helps reduce over-smoothing and improves the performance of GNN \citep{yang2020revisiting,zhao2019pairnorm}. See more connection between MPNN + VN and over-smoothing in \Cref{subsec:over-smoothing}. 


\subsection{Assumptions}
We have two mild assumptions on feature space $\mathcal{X}\subset \mb{R}^{n \times d}$ and the regularity of target function $\mL$.

\begin{assumption}\label{AS-2} $ \forall i\in [n], \vx_i \in \mc{X}_i,  \norm{\vx_i} < C_1$. This implies $\mc{X}$ is compact.  
\end{assumption} 

\begin{assumption}\label{AS-3}  $\norm{\mW_Q}< C_2, \norm{\mW_K} < C_2, \norm{\mW_V} < C_2$ for target layer $\mL$. Combined with AS\ref{AS-2} on $\mathcal{X}$, this means $\alpha'(\vx_i, \vx_j)$ is both upper and lower bounded, which further implies $\sum_j e^{\alpha'(\vx_i, \vx_j)}$ be both upper bounded and lower bounded. 
\end{assumption}

\section{$\mc{O}(1)$-depth $\mc{O}(1)$-width MPNN + VN for unbiased approximation of attention}
\label{sec:shallow-narrow-attention}
The standard self-attention takes $\mc{O}(n^2)$ computational time, therefore not scalable for large graphs. Reducing the computational complexity of self-attention in Transformer is active research \citep{tay2020efficient}. 
In this section, we consider self-attention in a specific type of efficient transformers, Performer \citep{choromanski2020rethinking} and Linear Transformer \citep{katharopoulos2020transformers}. 


One full self-attention layer $\mL$ is of the following form
\begin{equation}
\label{equ:attention-kernel-repr}
\vx_i^{(l+1)}=\sum_{j=1}^n \frac{\kappa\left(\mW_Q^{(l)} \vx_i^{(l)}, \mW_K^{(l)} \vx_j^{(l)}\right)}{\sum_{k=1}^n \kappa\left(\mW_Q^{(l)} \vx_i^{(l)}, \mW_K^{(l)} \vx_k^{(l)}\right)} \cdot\left(\mW_V^{(l)} \vx_j^{(l)}\right)
\end{equation}

where $\kappa: \mb{R}^d \times \mb{R}^d \rightarrow \mb{R}$ is the softmax kernel $\kappa(\vx, \vy):=\exp(\vx^T\vy)$. The kernel function can be approximated via $\kappa(\vx, \vy) = \ip{\Phi(\vx), \Phi(\vy)}_{\mc{V}} \approx \phi(\vx)^T\phi(\vy)$ where the first equation is by Mercer's theorem and $\phi(\cdot): \mb{R}^d \rightarrow \mb{R}^m $ is a low-dimensional feature map with random transformation. For Performer \citep{choromanski2020rethinking}, the choice of $\phi$ is taken as $\phi(\vx)=\frac{\exp \left(\frac{-\|\vx\|_2^2}{2}\right)}{\sqrt{m}}\left[\exp \left(\vw_1^{T} \vx\right), \cdots, \exp \left(\vw_m^{T} \vx\right)\right]$ where $\vw_k \sim \mathcal{N}\left(0, I_d\right)$ is i.i.d sampled random variable. For Linear Transformer \citep{katharopoulos2020transformers}, $\phi(\vx)=\operatorname{elu}(\vx)+1$. 

By switching $\kappa(\vx, \vy)$ to be $\phi(\vx)^T\phi(\vy)$, and denote $\vq_i=\mW_Q^{(l)} \vx_i^{(l)}, \vk_i=\mW_K^{(l)} \vx_i^{(l)} \text { and } \vv_i=\mW_V^{(l)} \vx_i^{(l)}$, the approximated version of \Cref{equ:attention-kernel-repr} by Performer and Linear Transformer becomes 
\begin{equation}
\begin{split}
\label{equ:modified-layer}
\vx_i^{(l+1)}&=\sum_{j=1}^n \frac{\phi\left(\vq_i\right)^{T} \phi\left(\vk_j\right)}{\sum_{k=1}^n \phi\left(\vq_i\right)^{T} \phi\left(\vk_k\right)} \cdot \vv_j \\
& =\frac{\left(\phi\left(\vq_i\right)^T \sum_{j=1}^n \phi\left(\vk_j\right) \otimes \vv_j\right)^T}{\phi\left(\vq_i\right)^{T} \sum_{k=1}^n \phi\left(\vk_k\right)}. \\
\end{split}
\end{equation}
where we use the matrix multiplication association rule to derive the second equality. 

The key advantage of \Cref{equ:modified-layer} is that $\sum_{j=1}^n \phi\left(\vk_j\right)$ and $\sum_{j=1}^n \phi(\vk_j) \otimes \vv_j$ can be approximated by the virtual node, and shared for all graph nodes, using only $\mc{O}(1)$ layers of MPNNs.  
We denote the self-attention layer of this form in \Cref{equ:modified-layer} as $\mL_{\text{Performer}}$. Linear Transformer differs from Performer by choosing a different form of $\phi(\vx)=\operatorname{Relu}(\vx)+1$ in its self-attention layer $\mL_{\text{Linear-Transformer}}$.  

In particular, the VN will approximate $\sum_{j=1}^n \phi\left(\vk_j\right)$ and $\sum_{j=1}^n \phi\left(\vk_j\right) \otimes \vv_j$, and represent it as its feature. Both $\phi\left(\vk_j\right)$ and $\phi\left(\vk_j\right) \otimes \vv_j$ can be approximated arbitrarily well by an MLP with constant width (constant in $n$ but can be exponential in $d$) and depth. Note that $\phi(\vk_j) \otimes \vv_j \in \mb{R}^{dm}$ but can be reshaped to 1 dimensional feature vector.

More specifically, the initial feature for the virtual node is $\bm{1}_{(d+1)m}$, where $d$ is the dimension of node features and $m$ is the number of random projections $\omega_i$.  
Message function + aggregation function for virtual node 
$\pool \phi_{\vngn}: \mb{R}^{(d+1)m} \times \mc{M} \rightarrow \mb{R}^{(d+1)m} $ is
\begin{equation}
\label{eq:vn-gn}
\begin{split}
 & \pool_{j \in [n]} \phi_{\vngn}^{(k)}(\cdot, \{\vx_i\}_i)  =  [\sum_{j=1}^n \phi\left(\vk_j\right), \\ & \tooned{\sum_{j=1}^n \phi\left(\vk_j\right) \otimes \vv_j}]
 \end{split}
\end{equation}
 where $\tooned{\cdot}$ flattens a 2D matrix to a 1D vector in raster order. This function can be arbitrarily approximated  by MLP. Note that the virtual node's feature dimension is $(d+1)m$ (where recall $m$ is the dimension of the feature map $\phi$ used in the linear transformer/Performer), which is larger than the dimension of the graph node $d$. This is consistent with the early intuition that the virtual node might be overloaded when passing information among nodes. 
 The update function for virtual node $\gamma_{\vn}: $ $\mb{R}^{(d+1)m} \times \mb{R}^{(d+1)m} \rightarrow \mb{R}^{(d+1)m}$ is just coping the second argument, which can be exactly implemented by MLP. 

VN then sends its message back to all other nodes, where each graph node $i$ applies the update function $\gamma_{\gn}: \mb{R}^{(d+1)m} \times \mb{R}^d \rightarrow \mb{R}^d$ of the form
\begin{equation}
\label{eq:gn}
\begin{split}
 & \gamma_{\gn} (\vx_i,  [\sum_{j=1}^n \phi\left(\vk_j\right), \tooned{\sum_{j=1}^n \phi\left(\vk_j\right) \otimes \vv_j}])\\ 
 & = \frac{\left(\phi\left(\vq_i\right) \sum_{j=1}^n \phi\left(\vk_j\right) \otimes \vv_j\right)^T}{\phi\left(\vq_i\right)^{T} \sum_{k=1}^n \phi\left(\vk_k\right)} 
 \end{split}
\end{equation}
 to update the graph node feature. 

As the update function $\gamma_{\gn}$ can not be computed exactly in MLP, what is left is to show that error induced by using MLP to approximate $\pool \phi_{\vngn}$ and $\gamma_{\gn}$ in \Cref{eq:vn-gn} and \Cref{eq:gn} can be made arbitrarily small. 

\begin{restatable}{theorem}{doubleconstant}
\label{thm:constant-depth-constant-width}
Under the AS\ref{AS-2} and AS\ref{AS-3}, MPNN + VN of $\mc{O}(1)$ width and $\mc{O}(1)$ depth can approximate $\mL_{\text{Performer}}$ and $\mL_{\text{Linear-Transformer}}$ arbitrarily well. 
\end{restatable}

\begin{proof}
We first prove the case of $\mL_{\text{Performer}}$.
We can decompose our target function as the composition of $\pool_{j \in [n]} \phi_{\vngn}^{(k)}$, $\gamma_{\gn}$ and $\phi$. 
By the uniform continuity of the functions, it suffices to show that 1) we can approximate $\phi$, 2) we can approximate operations in $\gamma_{\gn}$ and $\pool \phi_{\vngn}$ arbitrarily well on the compact domain, and 3) the denominator $\phi\left(\vq_i\right)^{T} \sum_{k=1}^n \phi\left(\vk_k\right)$ is uniformly lower bounded by a positive number for any node features in $\mc{X}$.

For 1), each component of $\phi$ is continuous and all inputs $\vk_j, \mathbf{q}_j$ lie in the compact domain so $\phi$ can be approximated arbitrarily well by MLP with $\mc{O}(1)$ width and $\mc{O}(1)$ depth \citep{cybenko1989approximation}. 

For 2), we need to approximate the operations in $\gamma_{\gn}$ and $\pool \phi_{\vngn}$, i.e., approximate multiplication, and vector-scalar division arbitrarily well.  As all those operations are continuous, it boils down to showing that all operands lie in a compact domain. By assumption AS\ref{AS-2} and AS\ref{AS-3} on $\mW_Q, \mW_K, \mW_V$ and input feature $\mc{X}$, we know that $\vq_i, \vk_i, \vv_i$ lies in a compact domain for all graph nodes $i$. As $\phi$ is continuous, this implies that $\phi(\vq_i), \sum_{j=1}^n \phi(\vk_j) \otimes \vv_j$ lies in a compact domain ($n$ is fixed), therefore the numerator lies in a compact domain. Lastly, since all operations do not involve $n$, the depth and width are constant in $n$.  

For 3), it is easy to see that $\phi\left(\vq_i\right)^{T} \sum_{k=1}^n \phi\left(\vk_k\right)$ is always positive.  We just need to show that the denominator is bound from below by a positive constant. For Performer, $\phi(\vx)=\frac{\exp \left(\frac{-\|\vx\|_2^2}{2}\right)}{\sqrt{m}}\left[\exp \left(\vw_1^{T} \vx\right), \cdots, \exp \left(\vw_m^{T} \vx\right)\right]$ where $\vw_k \sim \mathcal{N}\left(0, I_d\right)$. As all norm of input $\vx$ to $\phi$ is upper bounded by AS\ref{AS-2}, $\exp(\frac{-\|\vx\|_2^2}{2})$ is lower bounded. As $m$ is fixed, we know that $\norm{\vw^T_i \vx} \leq \norm{\vw_i} \norm{\vx}$, which implies that $\vw_i^T \vx$ is lower bounded by $-\norm{\vw_i} \norm{\vx}$ 
which further implies that $\exp(\vw_i^T \vx)$ is lower bounded. This means that $\phi\left(\vq_i\right)^{T} \sum_{k=1}^n \phi\left(\vk_k\right)$ is lower bounded. 

For Linear Transformer, the proof is essentially the same as above. We only need to show that $\phi(\vx)=\operatorname{elu}(\vx)+1$ is continuous and positive, which is indeed the case. 
\end{proof}

Besides Performers, there are many other different ways of obtaining linear complexity. In \Cref{subsec:ohter-linear-transformer}, we discuss the limitation of MPNN + VN on approximating other types of efficient transformers such as Linformer \citep{wang2020linformer} and Sparse Transformer \citep{child2019generating}.  

\section{$\mc{O}(1)$ depth $\mc{O}(n^d)$ width MPNN + VN}
\label{sec:shadow-wide-mpnn}

\begin{figure}[t!]
  \centering
  \includegraphics[width=1\linewidth]{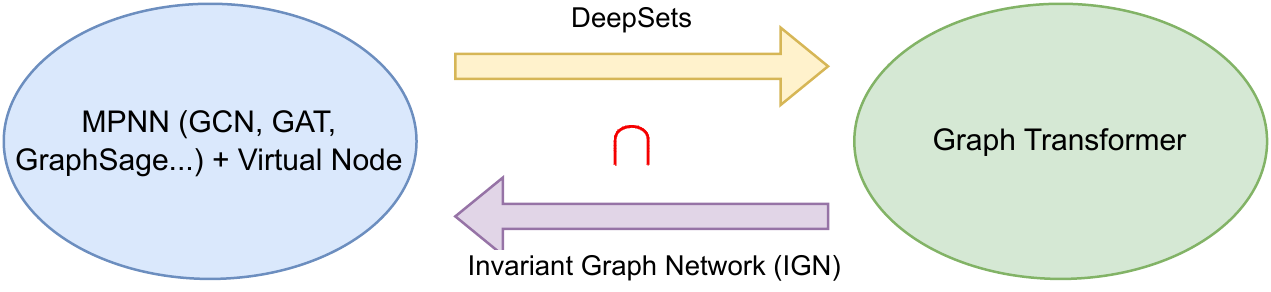}
\caption{The link between MPNN and GT is drawn via \DS{} in \Cref{sec:shadow-wide-mpnn} of our paper and Invariant Graph Network (IGN) in \citet{kim2022pure}. Interestingly, IGN is a generalization of \DS{} \citep{maron2018invariant}. }
\label{fig:mpnn+gt}
\end{figure}

We have shown that the MPNN + VN can approximate self-attention in Performer and Linear Transformer using only $\mc{O}(1)$ depth and $\mc{O}(1)$ width. 
One may naturally wonder whether MPNN + VN can approximate the self-attention layer in the \textit{full} transformer. In this section,
we show that MPNN + VN with $O(1)$ depth (number of layers), but with $\mc{O}(n^d)$ width, can approximate 1 self-attention layer (and full transformer) arbitrarily well. 

The main observation is that MPNN + VN is able to exactly simulate (not just approximate) equivariant \DS{} \citep{zaheer2017deep}, which is proved to be universal in approximating any permutation invariant/equivariant maps \citep{zaheer2017deep,segol2019universal}. Since the self-attention layer is permutation equivariant, this implies that MPNN + VN can approximate the self-attention layer (and full transformer) with $\mc{O}(1)$ depth and $\mc{O}(n^d)$ width following a result on \DS{} from \citet{segol2019universal}.

We first introduce the permutation equivariant map, equivariant \DS{}, and permutation equivariant universality.

\begin{definition}[permutation equivariant map]
A map $\mF: \mb{R}^{n \times k} \rightarrow \mb{R}^{n \times l}$ satisfying $\mF(\sigma \cdot \mX)=\sigma \cdot \mF(\mX)$ for all $\sigma \in S_n$ and $\mX \in \mb{R}^{n \times d}$ is called permutation equivariant.
\end{definition}

\begin{definition}[equivariant \DS{} of  \citet{zaheer2017deep}]
Equivariant \DS{} has the following form $\mF(\mX)=\mL_m^{\ds} \circ \nu \circ \cdots \circ \nu \circ \mL^{\ds}_1(\mX)$, where $\mL^{\ds}_i$ is a linear permutation equivariant layer and $\nu$ is a nonlinear layer such as ReLU. 
The linear permutation equivariant layer in \DS{} has the following form $\mL^{\ds}_i(\mX)=\mX \boldsymbol{A}+\frac{1}{n} \mathbf{1 1}^T \mX \mB+\mathbf{1} \boldsymbol{c}^T$, where $\mA, \mB \in \mb{R}^{d_{i} \times d_{i+1}}$, $\vc \in \mb{R}^{d_{i+1}}$ is the weights and bias in layer $i$, and $\nu$ is ReLU. 
\end{definition}

\begin{definition}[permutation equivariant universality]
\label{def:ds-universality}
Given a compact domain $\mc{X}$ of $\mb{R}^{n \times d_{\text{in}}}$, permutation equivariant universality of a model $\mF: \mb{R}^{n \times d_{\text{in}}} \rightarrow \mb{R}^{n \times d_{\text{out}}}$ means that for every permutation equivariant continuous function $\mH: \mb{R}^{n \times d_{\text{in}}} \rightarrow \mb{R}^{n \times d_{\text{out}}}$ defined over $\mc{X}$, and any $\epsilon>0$, there exists a choice of $m$ (i.e., network depth), $d_i$ (i.e., network width at layer $i$) and the trainable parameters of $\mF$ so that $\|\mH(\mX)-\mF(\mX)\|_{\infty}<\epsilon$ for all $\mX \in \mc{X}$.
\end{definition}

The universality of equivariant \DS{} is stated as follows. 

\begin{theorem}[\citet{segol2019universal}]
\label{thm:ds-universality}
\DS{} with constant layer is universal. Using ReLU activation the width $\omega:=\text{max}_i d_i $ ($d_i$ is the width for $i$-th layer of \DS{}) required for universal permutation equivariant
network satisfies $\omega \leq d_{\text{out}}+d_{\text{in}}+\left(\begin{array}{c} n+d_{\text{in}} \\ d_{\text{in}} \end{array}\right) = \mc{O}(n^{d_\text{in}})$.
\end{theorem}

We are now ready to state our main theorem. 

\begin{restatable}{theorem}{constantdepth}
\label{thm:constant-depth} 
MPNN + VN can simulate (not just approximate) equivariant \DS{}: $\mb{R}^{n \times d} \rightarrow \mb{R}^{n \times d}$. The depth and width of MPNN + VN needed to simulate \DS{} is up to a constant factor of the depth and width of \DS{}. 
This implies that MPNN + VN of $\mc{O}(1)$ depth and $\mc{O}(n^d)$ width is permutation equivariant universal, and can approximate self-attention layer and transformers arbitrarily well. 
\end{restatable}

\begin{table*}[t!]
    \caption{Baselines for \pepfunc (graph classification) and \pepstruct (graph regression). The performance metric is Average Precision (AP) for classification and MAE for regression. \textbf{Bold}: Best score.} 
    \label{tab:experiments_peptides}
    \tabtopvspace
    \begin{adjustwidth}{-2.5 cm}{-2.5 cm}\centering
   \scalebox{0.8}{    
    \setlength\tabcolsep{4pt} 
    \begin{tabular}{l c c c c c}\toprule
    \multirow{2}{*}{\textbf{Model}} & \multirow{2}{*}{\textbf{\# Params.}} & \multicolumn{2}{c}{\pepfunc} & \multicolumn{2}{c}{\pepstruct} \\ \cmidrule(lr){3-4} \cmidrule(lr){5-6}
                                    &                                      & \textbf{Test AP before VN}  & \textbf{Test AP after VN $\uparrow$}                                    & \textbf{Test MAE before VN} & \textbf{Test MAE after VN $\downarrow$} \\\midrule
    GCN                             & 508k                                 & 0.5930$\pm$0.0023            & 0.6623$\pm$0.0038                                                       & 0.3496$\pm$0.0013           & \first{0.2488$\pm$0.0021} \\
    GINE                            & 476k                                 & 0.5498$\pm$0.0079            & 0.6346$\pm$0.0071                                                       & 0.3547$\pm$0.0045           & 0.2584$\pm$0.0011 \\
    GatedGCN                        & 509k                                 & 0.5864$\pm$0.0077            & 0.6635$\pm$0.0024                                                       & 0.3420$\pm$0.0013           & 0.2523$\pm$0.0016 \\
    GatedGCN+RWSE                   & 506k                                 & 0.6069$\pm$0.0035            & \first{0.6685$\pm$0.0062}                                               & 0.3357$\pm$0.0006           & 0.2529$\pm$0.0009 \\ \midrule
    Transformer+LapPE               & 488k                                 & 0.6326$\pm$0.0126            & -                                                                       & 0.2529$\pm$0.0016           & - \\
    SAN+LapPE                       & 493k                                 & 0.6384$\pm$0.0121            & -                                                                       & 0.2683$\pm$0.0043           & - \\
    SAN+RWSE                        & 500k                                 & 0.6439$\pm$0.0075            & -                                                                       & 0.2545$\pm$0.0012           & - \\
    \bottomrule
    \end{tabular}
    }
    \end{adjustwidth}
\end{table*}

\begin{proof}
Equivariant \DS{} has the following form $\mF(\mX)=\mL^{\ds}_m \circ \nu \circ \cdots \circ \nu \circ \mL^{\ds}_1(\mX)$, where $\mL^{\ds}_i$ is the linear permutation equivariant layer and $\nu$ is an entrywise nonlinear activation layer.
Recall that the linear equivariant layer has the form $\mL^{\ds}_i(\mX)=\mX \boldsymbol{A}+\frac{1}{n} \mathbf{1 1}^T \mX \mB+\mathbf{1} \boldsymbol{c}^T$. 
As one can use the same nonlinear entrywise activation layer $\nu$ in MPNN + VN, it suffices to prove that MPNN + VN can compute linear permutation equivariant layer $\mL^{\ds}$. Now we show that 2 layers of MPNN + VN can exactly simulate any given linear permutation equivariant layer $\mL^{\ds}$.

Specifically, at layer 0, we initialized the node features as follows: The VN node feature is set to 0, while the node feature for the $i$-th graph node is set up as $\vx_i \in \mathbb{R}^d$.  

At layer 1: VN node feature is $\frac{1}{n} \mathbf{1 1}^T \mX$, average of node features. The collection of features over $n$ graph node feature is $\mX \mA$. 
We only need to transform graph node features by a linear transformation, and set the VN feature as the average of graph node features in the last iteration. Both can be exactly implemented in \Cref{def:simplified-hetero-mpnn-vn} of simplified heterogeneous MPNN + VN. 

At layer 2: VN node feature is set to be 0, and the graph node feature is $\mX \boldsymbol{A}+\frac{1}{n} \mathbf{1 1}^T \mX \mB+\mathbf{1} \boldsymbol{c}^T$. Here we only need to perform the matrix multiplication of the VN feature with $\mB$, as well as add a bias $\vc$. This can be done by implementing a linear function for $\gamma_{\gn}$. 

It is easy to see the width required for MPNN + VN to simulate \DS{} is constant.
Thus, one can use 2 layers of MPNN + VN to compute linear permutation equivariant layer $\mL^{\ds}_i$, which implies that MPNN + VN can simulate 1 layer of \DS{} exactly with constant depth and constant width (independent of $n$). Then by the universality of \DS{}, stated in \Cref{thm:ds-universality}, we conclude that MPNN + VN is also permutation equivariant universal, which implies that the constant layer of MPNN + VN with $\mc{O}(n^d)$ width is able to approximate any continuous equivariant maps. As the self-attention layer $\mL$ and full transformer are both continuous and equivariant, they can be approximated by MPNN + VN arbitrarily well.      
\end{proof}
Thanks to the connection between MPNN + VN with \DS{}, there is no extra assumption on $\mc{X}$ 
except for being compact. The drawback on the other hand is that the upper bound on the computational complexity needed to approximate the self-attention with wide MPNN + VN is worse than directly computing self-attention when $d>2$. 

\section{$\mc{O}(n)$ depth $\mc{O}(1)$ width MPNN + VN}
\label{sec:deep-vn}

The previous section shows that we can approximate a full attention layer in Transformer using MPNN with $\mc{O}(1)$ depth but $\mc{O}(n^d)$ width where $n$ is the number of nodes and $d$ is the dimension of node features. 
In practice, it is not desirable to have the width depend on the graph size. 

In this section, we hope to study MPNN + VNs with $\mc{O}(1)$ width and their ability to approximate a self-attention layer in the Transformer. However, this appears to be much more challenging. Our result in this section only shows that for a rather restrictive family of input graphs (see Assumption \ref{AS-1} below), we can approximate a full self-attention layer of transformer with an MPNN + VN of $\mc{O}(1)$ width and $\mc{O}(n)$ depth. We leave the question of MPNN + VN's ability in approximate transformers for more general families of graphs for future investigation.




We first introduce the notion of $(\mv,\delta)$ separable node features. This is needed to ensure that VN can approximately select one node feature to process at each iteration with attention $\alpha_{\vn}$, the self-attention in the virtual node. 




 \begin{definition}[$(\mv,\delta)$ separable by $\alpha$]
\label{vdeltaseparable}
Given a graph $G$ of size $n$ and a fixed $\mV \in \mb{R}^{n\times d} = [\vv_1, ..., \vv_n]$ 
and $\bar{\alpha} \in \mc{A}$, we say node feature $\mX\in \mb{R}^{n\times d}$ of $G$ is 
$(\mv,\delta)$ separable by some $\bar{\alpha}$ if the following holds. For any node feature $\vx_i$, there exist weights $\mW^{\bar{\alpha}}_K, \mW^{\bar{\alpha}}_Q$ in attention score $\bar{\alpha}$ 
such that $\bar{\alpha}(\vx_i, \vv_i) > \max_{j\neq i} \bar{\alpha}(\vx_j, \vv_i) + \delta$. We say set $\mathcal{X}$ is $(\mv,\delta)$ 
separable by $\bar{\alpha}$ if every element $\mX \in \mathcal{X}$ is $(\mv,\delta)$ separable by $\bar{\alpha}$.
 \end{definition}




\begin{table*}[t]
    \caption{Test performance in graph-level OGB benchmarks \citep{hu2020open}. Shown is the mean~$\pm$~s.d.~of 10 runs. 
    }
    \label{tab:results_ogb}
    \tabtopvspace
    \centering
    \fontsize{8.5pt}{8.5pt}\selectfont
    \begin{tabular}{lccccc}\toprule
    \multirow{2}{*}{\textbf{Model}} &\textbf{ogbg-molhiv} &\textbf{ogbg-molpcba} &\textbf{ogbg-ppa} &\textbf{ogbg-code2} \\\cmidrule{2-5}
                               & \textbf{AUROC $\uparrow$}    & \textbf{Avg.~Precision $\uparrow$} & \textbf{Accuracy $\uparrow$} & \textbf{F1 score $\uparrow$} \\\midrule
    GCN                        & 0.7606 $\pm$ 0.0097          & 0.2020 $\pm$ 0.0024          & 0.6839 $\pm$ 0.0084          & 0.1507 $\pm$ 0.0018 \\
    GCN+virtual node           & 0.7599 $\pm$ 0.0119          & 0.2424 $\pm$ 0.0034          & 0.6857 $\pm$ 0.0061          & 0.1595 $\pm$ 0.0018 \\
    GIN                        & 0.7558 $\pm$ 0.0140          & 0.2266 $\pm$ 0.0028          & 0.6892 $\pm$ 0.0100          & 0.1495 $\pm$ 0.0023 \\
    GIN+virtual node           & 0.7707 $\pm$ 0.0149          & 0.2703 $\pm$ 0.0023          & 0.7037 $\pm$ 0.0107          & 0.1581 $\pm$ 0.0026 \\
    \midrule
    SAN                        & 0.7785 $\pm$ 0.2470          & 0.2765 $\pm$ 0.0042          & --                           & -- \\
    GraphTrans (GCN-Virtual)   & --                           & 0.2761 $\pm$ 0.0029          & --                           & 0.1830 $\pm$ 0.0024 \\
    K-Subtree SAT              & --                           & --                           & 0.7522 $\pm$ 0.0056          & 0.1937 $\pm$ 0.0028 \\
    \method                    & 0.7880 $\pm$ 0.0101          & 0.2907 $\pm$ 0.0028          & 0.8015 $\pm$ 0.0033          & 0.1894 $\pm$ 0.0024 \\ \midrule
    MPNN + VN + NoPE           & 0.7676 $\pm$ 0.0172          & 0.2823 $\pm$ 0.0026          & 0.8055 $\pm$ 0.0038          & 0.1727 $\pm$ 0.0017             & \\
    MPNN + VN + PE             & 0.7687 $\pm$ 0.0136          & 0.2848 $\pm$ 0.0026          & 0.8027 $\pm$ 0.0026          & 0.1719 $\pm$ 0.0013             & \\
    \bottomrule
    \end{tabular}
\end{table*}

The use of $(\mv,\delta)$ separability is to approximate hard selection function arbitrarily well, which is stated below and proved in \Cref{subsec:assumptions}. 
\begin{restatable}[approximate hard selection]{lemma}{UniformSelection}
\label{lemma-uniform-selection}
Given $\mathcal{X}$ is $(\mv,\delta)$ separable by $\bar{\alpha} $ for some fixed $\mv\in \mb{R}^{n\times d}$, $\bar{\alpha} \in \mc{A}$ and $\delta > 0$, the following holds. For any $\epsilon>0$ and $i\in [n]$, there exists a set of attention weights $\mW_{i, Q}, \mW_{i, K}$ in $i$-th layer of MPNN + VN such that $\alpha_{\vn}(\vx_i, \vv_i ) > 1-\epsilon$ for any $\vx_i \in \mathcal{X}_i$. In other words, we can approximate a hard selection function $f_i(\vx_1, ..., \vx_n) = \vx_i$ arbitrarily well on $\mathcal{X}$ by setting $\alpha_{\vn} = \bar{\alpha}$.  
\end{restatable}


With the notation set up, We now state an extra assumption needed for deep MPNN + VN case and the main theorem. 
\begin{assumption}\label{AS-1} 
$\mathcal{X}$ is $(\mv,\delta)$ separable by $\bar{\alpha}$ for some fixed $\mv \in \mb{R}^{n\times d}$, $\bar{\alpha}\in \mc{A}$ and $\delta>0$. 
\end{assumption}

\begin{restatable}{theorem}{mainthm}
\label{thm:constant-width}
Assume AS 1-3 hold for the compact set $\mc{X}$ and $\mL$. Given any graph $G$ of size $n$ with node features $\mX\in \mc{X}$, and a self-attention layer $\mL$ on $G$ (fix $\mW_K, \mW_Q, \mW_V$ in $\alpha$), there exists a $\mc{O}(n)$ layer of heterogeneous MPNN + VN with the specific aggregate/update/message function that can approximate $\mL$ on $\mc{X}$ arbitrarily well. 
\end{restatable}

The proof is presented in the \Cref{sec:approximate-full-self-attention}. On the high level, we can design an MPNN + VN where the $i$-th layer will select $\tilde{\vx}_i$, an approximation of $\vx_i$ via attention mechanism, enabled by \Cref{lemma-uniform-selection}, and send $\tilde{\vx}_i$ to the virtual node. Virtual node will then pass the $\tilde{\vx}_i$ to all graph nodes and computes the approximation of $e^{\alpha(\vx_i, \vx_j)}, \forall j\in [n]$. Repeat such procedures $n$ times for all graph nodes, and finally, use the last layer for attention normalization. A slight relaxation of AS\ref{AS-1} is also provided in the appendix. 



\begin{table*}[th!]
    \caption{Evaluation on PCQM4Mv2~\cite{hu2021ogb} dataset.
    For \method evaluation, we treated the \emph{validation} set of the dataset as a test set, since the \emph{test-dev} set labels are private.
    }
    \label{tab:results_pcqm4m}
    \centering
    \fontsize{8.0pt}{8.0pt}\selectfont
    \begin{tabular}{lccccc}\toprule
    \label{tab:pcqm}
    \multirow{2}{*}{\textbf{Model}} &\multicolumn{3}{c}{\textbf{PCQM4Mv2}} \\\cmidrule{2-5}
    &\textbf{Test-dev MAE $\downarrow$} &\textbf{Validation MAE $\downarrow$} &\textbf{Training MAE} &\textbf{\# Param.} \\\midrule
    GCN &0.1398 &0.1379 & n/a &2.0M \\
    GCN-virtual &0.1152 &0.1153 & n/a &4.9M \\
    GIN &0.1218 &0.1195 & n/a &3.8M \\
    GIN-virtual &0.1084 &0.1083 & n/a &6.7M \\\midrule
    GRPE~\citep{park2022grpe} &0.0898 &0.0890 & n/a &46.2M \\
    EGT~\citep{hussain2022global} &0.0872 &0.0869 & n/a &89.3M \\
    Graphormer~\citep{shi2022benchmarking} &n/a &0.0864 &0.0348 & 48.3M \\
    \method-small &n/a &0.0938 &0.0653 &6.2M \\
    \method-medium &n/a &0.0858 &0.0726  &19.4M \\ \midrule
    MPNN + VN + PE (small)  &n/a &0.0942 &0.0617 &5.2M \\
    MPNN + VN + PE (medium)  &n/a &0.0867 &0.0703 &16.4M \\
    MPNN + VN + NoPE (small)  &n/a &0.0967 &0.0576 &5.2M \\
    MPNN + VN + NoPE (medium)  &n/a &0.0889 &0.0693 &16.4M \\
    \bottomrule
    \end{tabular}
\end{table*}

\section{Experiments}
We benchmark MPNN + VN for three tasks, long range interaction modeling in \Cref{subsec:lrgb}, OGB regression tasks in \Cref{subsec:ogb}, and focasting sea surface temperature in \Cref{sec:climate}. The code is available \url{https://github.com/Chen-Cai-OSU/MPNN-GT-Connection}. 

\subsection{MPNN + VN for LRGB Datasets}
\label{subsec:lrgb}
We experiment with MPNN + VN for Long Range Graph Benchmark (LRGB) datasets. 
Original paper \citep{dwivedi2022long} observes that GT outperforms MPNN on 4 out of 5 datasets, among which GT shows significant improvement over MPNN on \pepfunc and \pepstruct for all MPNNs. To test the effectiveness of the virtual node, we take the original code and modify the graph topology by adding a virtual node and keeping the hyperparameters of all models unchanged. 

Results are in \Cref{tab:experiments_peptides}. Interestingly, such a simple change can boost MPNN + VN by a large margin on \pepfunc and \pepstruct. Notably, with the addition of VN, GatedGCN + RWSE (random-walk structural encoding) after augmented by VN {\bf outperforms all transformers} on \pepfunc, and GCN outperforms transformers on \pepstruct. 

\subsection{Stronger MPNN + VN Implementation}
\label{subsec:ogb}
Next, by leveraging the modularized implementation from GraphGPS \citep{rampavsek2022recipe}, we implemented a version of MPNN + VN with/without extra positional embedding. Our goal is not to achieve SOTA but instead to push the limit of MPNN + VN and better understand the source of the performance gain for GT. 
In particular, we replace the GlobalAttention Module in GraphGPS with \DS{}, which is equivalent to one specific version of MPNN + VN. We tested this specific version of MPNN + VN on 4 OGB datasets, both with and without the use of positional embedding. 
The results are reported in Table \ref{tab:results_ogb}. Interestingly, even without the extra position embedding, our MPNN + VN is able to further improve over the previous GCN + VN \& GIN + VN implementation. 
The improvement on \textbf{ogbg-ppa} is particularly impressive, which is from 0.7037 to 0.8055. 
Furthermore, it is important to note that while MPNN + VN does not necessarily outperform GraphGPS, which is a state-of-the-art architecture using both MPNN, Position/structure encoding and Transformer, the difference is quite small -- this however, is achieved by a simple MPNN + VN architecture. 

We also test MPNN + VN on large-scale molecule datasets PCQMv2, which has 529,434 molecule graphs. 
We followed \citep{rampavsek2022recipe}
and used the original validation set as the test set, while we left out random 150K molecules for our validation set. 
As we can see from \Cref{tab:pcqm}, MPNN + VN + NoPE performs significantly better than the early MPNN + VN implementation: GIN + VN and GCN + VN. 
The performance gap between GPS on the other hand is rather small: 0.0938 (GPS) vs. 0.0942 (MPNN + VN + PE) for the small model and 0.0858 (GPS) vs. 0.0867 (MPNN + VN + PE) for the medium model. 

\subsection{Forecasting Sea Surface Temperature} \label{sec:climate}

In this experiment, we apply our MPNN + VN model to forecast sea surface temperature (SST). We are particularly interested in the empirical comparison between MPNN + VN and Linear Transformer \citep{katharopoulos-et-al-2020} as according to \Cref{sec:shallow-narrow-attention}, MPNN + VN theoretically can approximate Linear Transformer.

In particular, from the DOISST data proposed by \citep{ImprovementsoftheDailyOptimumInterpolationSeaSurfaceTemperatureDOISSTVersion21}, we construct a dataset of daily SST in the Pacific Ocean from 1982 to 2021, in the region of longitudes from $180.125^\circ\text{E}$ to $269.875^\circ\text{E}$ and latitudes from $-14.875^\circ\text{N}$ to $14.875^\circ\text{N}$. Following the procedure from \citep{de2018deep,deBezenac2019} and \citet{wang2022metalearning}, we divide the region into 11 batches of equal size with 30 longitudes and 30 latitudes at 0.5$^\circ$-degree resolution, that can be represented as a graph of 900 nodes. 
The tasks are to predict the next 4 weeks, 2 weeks and 1 week of SST at each location, given 6 weeks of historical data. We train on data from years 1982--2018, validate on data from 2019 and test on data from 2020--2021. The number of training, validation, and testing examples are roughly 150K, 3K, and 7K. See details of dataset construction, model architectures, and training scheme in \Cref{appendix:climate}. 

We compare our model to other baselines including TF-Net \cite{Rui2020}, a SOTA method for spatiotemporal forecasting, Linear Transformer \citep{katharopoulos-et-al-2020,wang2020linformer} with Laplacian positional encoding (LapPE), and Multilayer Perceptron (MLP). We use Mean Square Error (MSE) as the metric and report the errors on the test set, shown in the \Cref{tab:climate}. 
We observe that the virtual node (VN) alone improves upon MPNN by $3.8\%$, $6.6\%$ and $4.5\%$ in 4-, 2- and 1-week settings, respectively. Furthermore, aligned with our theory in \Cref{sec:shallow-narrow-attention}, MPNN + VN indeed achieves comparable results with Linear Transformer and outperforms it by a margin of $0.4\%$, $2.8\%$ and $4.3\%$ in 4-, 2- and 1-week settings, respectively.


\section{Concluding Remarks}
In this paper, we study the expressive power of MPNN + VN under the lens of GT. If we target the self-attention layer in Performer and Linear Transformer, one only needs $\mc{O}(1)$-depth $\mc{O}(1)$ width for arbitrary approximation error. 
For self-attention in full transformer, we prove that heterogeneous MPNN + VN of either $\mc{O}(1)$ depth $\mc{O}(n^d)$ width or $\mc{O}(n)$ depth $\mc{O}(1)$ width (under some assumptions) can approximate 1 self-attention layer arbitrarily well. 
Compared to early results \citep{kim2022pure} showing GT can approximate MPNN, our theoretical result draws the connection from the inverse direction. 


\begin{table}
\label{tab:sst}
\caption{\label{tab:climate} Results of SST prediction.} 
\begin{center}
\scalebox{0.9}{
\begin{tabular}{lccc}
\midrule
\textbf{Model} & \textbf{4 weeks} & \textbf{2 weeks} & \textbf{1 week} \\
\midrule
MLP & 0.3302 & 0.2710 & 0.2121 \\
TF-Net  & 0.2833 & \textbf{0.2036} & \textbf{0.1462} \\
Linear Transformer + LapPE & 0.2818 & 0.2191 & 0.1610 \\
MPNN & 0.2917 & 0.2281 & 0.1613 \\
\midrule
MPNN + VN & \textbf{0.2806} & 0.2130 & 0.1540 \\
\bottomrule
\end{tabular}
}
\end{center}
\end{table}
On the empirical side, we demonstrate that MPNN + VN remains a surprisingly strong baseline. Despite recent efforts, we still lack good benchmark datasets where GT can outperform MPNN by a large margin. Understanding the inductive bias of MPNN and GT remains challenging. For example, can we mathematically characterize tasks that require effective long-range interaction modeling, and provide a theoretical justification for using GT over MPNN (or vice versa) for certain classes of functions on the space of graphs? We believe making processes towards answering such questions is an important future direction for the graph learning community.    

\section*{Acknowledgement}
This work was supported in part by the U.S. Department Of Energy, Office of Science, U. S. Army Research Office under Grant W911NF-20-1-0334, Google Faculty Award, Amazon Research Award, a Qualcomm gift fund, and NSF Grants \#2134274, \#2107256, \#2134178, CCF-2217033, and CCF-2112665.

\bibliography{./main}
\bibliographystyle{icml2023}

\newpage
\onecolumn
\appendix
\section{Notations} 
We provide a notation table for references. 
\begin{table}[!htbp]
\caption{Summary of important notations.} 
\begin{center}
\scalebox{0.9}{
{
\begin{tabular}{@{}l|l@{}}
    \hline
    \toprule
    Symbol & Meaning \\
    \midrule
    \midrule
    $\mX \in \mc{X} \subset \R{n \times d}$ & graph node features \\
    $\vx_i \in \R{1\times d}$ & graph node $i$'s feature \\
    $\tilde{\vx}_i \in \R{1\times d}$ & approximated graph node $i$'s feature via attention selection \\
    $\mc{M}$ & A multiset of vectors in $\R{d}$  \\
    $\mW_Q^{(l)}, \mW_K^{(l)}, \mW_V^{(l)} \in \R{d\times d'}$ & attention matrix of $l$-th self-attention layer in graph transformer \\
    $\mc{X}$ & feature space \\ 
    $\mc{X}_i$ & projection of feature space onto $i$-th coordinate \\
    $\mL^{\ds}_i$ & $i$-th linear permutation equivariant layer in \DS{} \\ 
    $\mL, \mL'$ & full self attention layer; approximate self attention layer in Performer \\  
    $\vz_{\vn}^{(l)}, \vz_{i}^{(l)}$ &  virtual/graph node feature at layer $l$ of heterogeneous MPNN + VN  \\
    $\alpha_{\vn}$ & attention score in MPNN + VN \\
    
    $\alpha(\cdot, \cdot)$ & normalized attention score \\
    $\alpha_{\gat}(\cdot, \cdot)$ & normalized attention score with \gat{} \\
    $\alpha'(\cdot, \cdot)$ & unnormalized attention score. $\alpha'(\vu, \vv) = \vu \mW_Q(\mW_K)^T\vv^T$ \\
    $\alpha'_{\gat}(\cdot, \cdot)$ & unnormalized attention score with \gat{}. $\alpha'_{\text{GATv2}}(\vu, \vv) :=  \va^T \operatorname{LeakyReLU}\left(\mW \cdot\left[\vu \| \vv\right] + \vb \right)$ \\
    $\mc{A}$ & space of attentions, where each element $\alpha \in \mc{A}$ is of form $\alpha(\vu, \vv) = \text{softmax}(\vu \mW_Q(\mW_K)^T\vv^T)$ \\
    $C_1$ & upper bound on norm of all node features $\norm{\vx_i}$ \\
    $C_2$ & upper bound on the norm of $\mW_Q, \mW_K, \mW_V$ in target $\mL$ \\
    $C_3$ & upper bound on the norm of attention weights of $\alpha_{\vn}$ when selecting $\vx_i$ \\
    \midrule
    $\gamma^{(k)}(\cdot, \cdot)$ & update function \\
    $\theta^{(k)}(\cdot, \cdot)$ & message function \\
    $\pool(\cdot)$ & aggregation function \\
    \bottomrule
\end{tabular}
}
}
\end{center}
\label{table:symbol_notation}
\end{table}


\section{$\mc{O}(n)$ Heterogeneous MPNN + VN Layer with $O(1)$ Width Can Approximate $1$ Self Attention Layer Arbitrarily Well}
\label{sec:approximate-full-self-attention}

\subsection{Assumptions}
\label{subsec:assumptions}

A special case of $(\mv,\delta)$ separable is when $\delta=0$, i.e., $\forall i, \bar{\alpha}(\vx_i, \vv_i) > \max_{j\neq i} \bar{\alpha}(\vx_j, \vv_i)$. We provide a geometric characterization of $\mX$ being $(\mv, 0)$ separable. 
\begin{restatable}{lemma}{Geometriccharacterization}
\label{lemma:Geometric-characterization} Given $\bar{\alpha}$ and $\mV$, $\mX$ is $(\mv, 0)$ separable by $\bar{\alpha}$  $\Longleftrightarrow$ $\vx_i$ is not in the convex hull spanned by $\{\vx_j\}_{j\neq i}$. $\Longleftrightarrow$ there are no points in the convex hull of $\{\vx_i\}_{i\in[n]}$.
\end{restatable}

\begin{proof}
The second equivalence is trivial so we only prove the first equivalence. By definition, $\mX$ is $(\mv, 0)$ separable by $\bar{\alpha}$ $\Longleftrightarrow$ $\bar{\alpha}(\vx_i, \vv_i) > \max_{j\neq i} \bar{\alpha}(\vx_j, \vv_i) \forall i \in [n]$ $\Longleftrightarrow$ $\ip{\vx_i, \mW_Q^{\bar{\alpha}} \mW_K^{{\bar{\alpha}, T}} \vv_i} > \max_{j\neq i} \ip{\vx_j, \mW_Q^{\bar{\alpha}} \mW_K^{{\bar{\alpha}}, T} \vv_i} \forall i \in [n]$.

By denoting the $\vv'_i := \mW_Q^{\bar{\alpha}}\mW_K^{{\bar{\alpha}}, T} \vv_i \in \mb{R}^{d}$, we know that $\ip{\vx_i, \vv'_i} > \max_{j\neq i} \ip{\vx_j, \vv'_i} \forall i \in [n]$, which implies that  
$\forall i \in [n], \vx_i$ can be linearly seprated from $\{\vx_j\}_{j\neq i}$ $\Longleftrightarrow$ $\vx_i$ is not in the convex hull spanned by $\{\vx_j\}_{j\neq i}$, which concludes the proof. 
\end{proof}

\UniformSelection*
\begin{proof}
Denote $\bar{\alpha}'$ as the unnormalized $\bar{\alpha}$. As $\mathcal{X}$ is $(\mv,\delta)$ separable by $\bar{\alpha}$, by definition we know that $\bar{\alpha}(\vx_i, \vv_i) > \max_{j\neq i }\bar{\alpha}(\vx_j, \vv_i) + \delta$ holds for any $i\in [n]$ and $\vx_i \in \mc{M}$. We can amplify this by multiple the weight matrix in $\bar{\alpha}$ by a constant factor $c$ to make $\bar{\alpha}'(\vx_i, \vv_i)> \max_{j\neq i}\bar{\alpha}'(\vx_j, \vv_i) + c\delta$. This implies that $e^{\bar{\alpha}'(\vx_i, \vv_i)}> e^ {c\delta} \max_{j\neq i}e^{\bar{\alpha}'(\vx_j, \vv_i)}$. This means after softmax, the attention score $\bar{\alpha}(\vx_i, \vv_i)$ will be at least $\frac{e^{c \delta}}{e^{c \delta}+n-1}$. We can pick a large enough $c(\delta, \epsilon)$ such that $\bar{\alpha}(\vx_i, \vv_i) > 1-\epsilon$ for any $\vx_i \in \mathcal{X}_i$ and $\epsilon > 0$.
\end{proof}

\textbf{Proof Intuition and Outline.} On the high level, $i$-th MPNN + VN layer will select $\tilde{\vx}_i$, an approximation $i$-th node feature $\vx_i$ via attention mechanism, enabled by \Cref{lemma-uniform-selection}, and send $\tilde{\vx}_i$ to the virtual node. Virtual node will then pass the $\tilde{\vx}_i$ to all graph nodes and computes the approximation of $e^{\alpha(\vx_i, \vx_j)}, \forall j\in [n]$. Repeat such procedures $n$ times for all graph nodes, and finally, use the last layer for attention normalization. 

The main challenge of the proof is to 1) come up with message/update/aggregation functions for heterogeneous MPNN + VN layer, which is shown in \Cref{subsec-mpnn-form},
and 2) ensure the approximation error, both from approximating Aggregate/Message/Update function with MLP and the noisy input, can be well controlled, which is proved in \Cref{subsec-controling-error}.

We will first instantiate the Aggregate/Message/Update function for virtual/graph nodes in \Cref{subsec-mpnn-form}, and prove that each component can be either exactly computed or approximated to an arbitrary degree by MLP. Then we go through an example in \Cref{subsec-a-running-example} of approximate self-attention layer $\mL$ with $\mc{O}(n)$ MPNN + VN layers. The main proof is presented in \Cref{subsec-controling-error}, where we show that the approximation error introduced during different steps is well controlled. Lastly, in \Cref{subsec:relax-assumption} we show assumption on node features can be relaxed if a more powerful attention mechanism \gatii \citep{brody2021attentive} is allowed in MPNN + VN.  

\subsection{Aggregate/Message/Update
Functions}
\label{subsec-mpnn-form}

Let $\mc{M}$ be a multiset of vectors in $\R{d}$. 
The specific form of Aggregate/Message/Update for virtual and graph nodes are listed below. Note that ideal forms will be implemented as MLP, which will incur an approximation error that can be controlled to an arbitrary degree. We use $\vz_{\vn}^{(k)}$ denotes the virtual node's feature at $l$-th layer, and $\vz_{i}^{(k)}$ denotes the graph node $i$'s node feature. Iteration index $k$ starts with 0 and the node index starts with 1. 

\subsubsection{virtual node}\label{subsubsec-vn}
At $k$-th iteration, virtual node $i$'s feature $\vz_i^{(k)}$ is a concatenation of three component $[\tilde{\vx}_i, \vv_{k+1}, 0]$ where the first component is the approximately selected node features $\vx_i\in \R{d}$, the second component is the $\vv_i\in \R{d}$ that is used to select the node feature in $i$-th iteration. The last component is just a placeholder to ensure the dimension of the virtual node and graph node are the same. It is introduced to simplify notation.  


\emph{Initial feature} is $\vz_{\vn}^{(0)} = [\bm{0}_d, \vv_1, 0]$. 

\emph{Message function + Aggregation function} $\pool_{j \in [n]} \phi_{\vngn}^{(k)}: \mb{R}^{2d+1} \times \mathcal{M} \rightarrow \mb{R}^{2d+1}$ has two cases to discuss depending on value of $k$. For $k=1, 2, ..., n$,
\begin{equation}\label{equ:vn-gn}
\begin{split}
 & \pool_{j \in [n]} \phi_{\vngn}^{(k)}(\vz_{\vn}^{(k-1)}, \{\vz_{i}^{(k-1)}\}_i) = \\
  & \begin{cases} 
  \sum_i \alpha_{\vn}(\vz_{\vn}^{(k-1)}, \vz_i^{(k-1)})\vz_i^{(k-1)} & k = 1, 2, ..., n \\
  \bm{1}_{2d+1} & k = n+1, n+2 \\ 
  \end{cases}
\end{split}
\end{equation}

 where $\vz_{\vn}^{(k-1)} = [\tilde{\vx}_{k-1}, \vv_{k}, 0]$. $\vz_i^{(k-1)} = [\overbrace{\underbrace{\vx_i}_{d \text{ dim}}, ..., ...}^{2d+1 \text{ dim}}]$ is the node $i$'s feature, where the first $d$ coordinates remain fixed for different iteration $k$. 
$\pool_{j \in [n]} \phi_{\vngn}^{(k)}$ use attention $\alpha_{\vn}$ to approximately select $k$-th node feature $[\overbrace{\underbrace{\vx_k}_{d \text{ dim}}, ..., ...}^{2d+1 \text{ dim}}]$. 
Note that the particular form of attention $\alpha_{\vn}$ needed for soft selection is not important as long as we can approximate hard selection arbitrarily well. As the $\vz^{(k-1)}_{\vn}$ contains $\vv_k$ and $\vz_i^{(k-1)}$ contains $\vx_i$ (see definition of graph node feature in \Cref{subsubsec-gn}),  this step can be made as close to hard selection as possible, according to \Cref{lemma-approximation-node-feature}. 

In the case of $k=n+1$,
$\pool_{j \in [n]} \phi_{\vngn}^{(k)}: \underbrace{\mb{R}^{2d+1}}_{\vn} \times \underbrace{\mathcal{M}}_{\text{set of } \gn} \rightarrow \mb{R}^d$ simply returns $\bm{1}_{2d+1}$. This can be exactly implemented by an MLP. 

\emph{Update function $\gamma_{\vn}^{(k)}: \underbrace{\mb{R}^{2d+1}}_{\vn} \times \underbrace{\mb{R}^{2d+1}}_{\gn} \rightarrow \mb{R}^{2d+1}$}:
Given the virtual node's feature in the last iteration, and the selected feature in virtual node $\vy = [\vx_{k}, ..., ...]$ with $\alpha_{\vn}$,
\begin{equation}
\gamma_{\vn}^{(k)}(\cdot, \vy) = 
\begin{cases}
[\vy_{0:d}, \vv_{k+1}, 0] & k=1, ..., n-1 \\
[\vy_{0:d}, \bm{0}_d, 0] & k=n \\
\bm{1}_{2d+1}                         & k=n+1, n+2  \\
\end{cases}
\end{equation}
where $\vy_{0:d}$ denotes the first $d$ channels of $\vy\in \mb{R}^{2d+1}$. $\vy$ denotes the selected node $\vz_i$'s feature in Message/Aggregation function. 
$\gamma_{\vn}^{(k)}$ can be exactly implemented by an MLP for any $k=1, ..., n+2$.

\subsubsection{Graph node}\label{subsubsec-gn}
Graph node $i$'s feature $\vv_i \in \R{2d+1}$ can be thought of as a concatenation of three components $[\underbrace{\vx_i}_{d \text{ dim}}, \underbrace{\tmp{}}_{d \text{ dim}}, \underbrace{\ps{}}_{1 \text{ dim}}]$, where $\vx_i, \in \R{d}, \tmp{} \in \R{d}$ \footnote{\tmp{} technicially denotes the dimension of projected feature by $W_V$ and does not has to be in $\R{d}$. We use $\R{d}$ here to reduce the notation clutter.}, and $\ps{}\in \R{}$. 

In particular, $\vx_i$ is the initial node feature. The first $d$ channel will stay the same until the layer $n+2$. $\tmp{} = \sum_{j \in \text{subset of} [n] } e^{\alpha'_{ij}}\vx_j$ stands for the unnormalized attention contribution up to the current iteration. $\ps{} \in \R{}$ is a partial sum of the unnormalized attention score, which will be used for normalization in the $n+2$-th iteration. 

\emph{Initial feature} $\vz_{\gn}^{(0)} =[\vx_i, \bm{0}_{d}, 0]$.

\emph{Message function + Aggregate function:
$\pool_{j \in [n]} \phi_{\gnvn}^{(k)}: \mb{R}^{2d+1}\times \mb{R}^{2d+1} \rightarrow \mb{R}^{2d+1}$}
is just ``copying the second argument'' since there is just one incoming message from the virtual node, i.e., $\pool_{j \in [n]} \phi_{\gnvn}^{(k)}(\vx, \{\vy\}) = \vy$. This function can be exactly implemented by an MLP. 

\emph{Update function}
$\gamma_{\gn}^{(k)}: \underbrace{\mb{R}^{2d+1}}_{\gn} \times \underbrace{\mb{R}^{2d+1}}_{\vn} \rightarrow \mb{R}^{2d+1}$ is of the following form. 
\begin{equation}\label{eqn-gn-update-function}
\begin{split}
& \gamma_{\gn}^{(k)}([\vx, \tmp{}, \ps{}], \vy) = \\
& \begin{cases}
[\vx, \tmp{}, \ps{}] &k = 1 \\
[\vx, \tmp{} + e^{\alpha'(\vx, \vy_{0:d})}\mW_V \vy_{0:d}, \\ \ps{}+e^{\alpha'(\vx, \vy_{0:d})}] &k = 2, ..., n+1  \\
[\frac{\tmp{}}{\ps{}}, \bm{0}_d, 0] & k=n+2 \\
\end{cases}
\end{split}
\end{equation}
where $\alpha'(\vx, \vy_{0:d})$ is the usual unnormalized attention score. Update function $\gamma_{\gn}^{(k)}$ can be arbitrarily approximated by an MLP, which is proved below. 

\begin{restatable}{lemma}{LemmaApproximateUpdateFunction}
\label{lemma-approximate-update-function}
Update function $\gamma_{\gn}^{(k)}$ can be arbitrarily approximated by an MLP from $\R{2d+1} \times \R{2d+1}$ to $\R{2d+1}$ for all $k=1,..., n+2$.
\end{restatable}
\begin{proof}
We will show that for any $k = 1, ..., n+2$, the target function
$\gamma_{\gn}^{(k)}: \mb{R}^{2d+1}\times \mb{R}^{2d+1} \rightarrow \mb{R}^{2d+1}$ is continuous and the domain is compact. By the universality of MLP in approximating continuous function on the compact domain, we know $\gamma_{\gn}^{(k)}$ can be approximated to arbitrary precision by an MLP. 

Recall that
\begin{equation*}
\begin{split}
& \gamma_{\gn}^{(k)}([\vx, \tmp{}, \ps{}], \vy) = \\
& \begin{cases}
[\vx, \tmp{}, \ps{}] &k = 1 \\
[\vx, \tmp{} + e^{\alpha'(\vx, \vy_{0:d})}\mW_V \vy_{0:d}, \\ \ps{}+e^{\alpha'(\vx, \vy_{0:d})}] &k = 2, ..., n+1  \\
[\frac{\tmp{}}{\ps{}}, \bm{0}_d, 0] & k=n+2 \\
\end{cases}
\end{split}
\end{equation*}
it is easy to see that $k=1$, $\gamma_{\gn}^{(1)}$ is continuous. We next show for $k=2, ..., n+2$, $\gamma_{\gn}^{(1)}$ is also continuous and all arguments lie in a compact domain. 

$\gamma_{\gn}^{(k)}$ is continuous because to a) $\alpha'(\vx, \vy)$ is continuous b) scalar-vector multiplication, sum, and exponential are all continuous. Next, we show that four component $\vx, \tmp{}, \ps{}, \vy_{0:d}$ all lies in a compact domain.

$\vx$ is the initial node features, and by AS\ref{AS-2} their norm is bounded so $\vx$ is in a compact domain.

$\tmp{}$ is an approximation of $e^{\alpha'_{i, 1}}\mW_V \vx_1 + e^{\alpha'_{i, 2}}\mW_V \vx_2+...$. As $\alpha'(\vx_i, \vx_j)$ is both upper and lower bounded by AS\ref{AS-3} for all $i, j \in [n]$ and $\vx_i$ is bounded by AS\ref{AS-2}, $e^{\alpha'_{i, 1}}\mW_V \vx_1 + e^{\alpha'_{i, 2}}\mW_V \vx_2+...$ is also bounded from below and above. $\tmp{}$ will also be bounded as we can control the error to any precision. 

$\ps{}$ is an approximation of $e^{\alpha'_{i, 1}} + e^{\alpha'_{i, 2}} + ...$. For the same reason as the case above, $\ps{}$ is also bounded both below and above.

$\vy_{0:d}$ will be $\tilde{\vx}_i$ at $i$-th iteration so it will also be bounded by AS\ref{AS-2}.

Therefore we conclude the proof. 
\end{proof}

\subsection{A Running Example}
\label{subsec-a-running-example}
We provide an example to illustrate how node features are updated in each iteration. 

\textbf{Time $0$}: 
All nodes are initialized as indicated in \Cref{subsec-mpnn-form}. Virtual node feature  $\vz_{\vn}^{(0)} = [\bm{0}_d, \vv_1, 0]$. Graph node feature $\vz_{i}^{(0)} = [\vx_i, \bm{0}_{d}, 0]$ for all $i\in[n]$.

\textbf{Time $1$}:

For virtual node, according to the definition of $\pool_{j \in [n]} \phi_{\vngn}^{(1)}$ in \Cref{equ:vn-gn}, it will pick an approximation of $\vx_1$, i.e. $\tilde{\vx}_1$. Note that the approximation error can be made arbitrarily small. VN's node feature $\vz_{\vn}^{(1)} = [\tilde{\vx}_1, \vv_2, 0]$. 

For $i$-th graph node feature,  $\vz_{\vn}^{(0)} = \bm{1}_{d}$, and $\vz_i^{(0)} = [\vx_i, \bm{0}_{d}, 0]$. According to $\gamma_{\gn}^{(k)}$ in  \Cref{eqn-gn-update-function}, $\vz_i^{(1)} = [\vx_i, \bm{0}_{d}, 0]$.

\textbf{Time $2$}:

For the virtual node feature: similar to the analysis in time 1, VN's feature $\vz_{\vn}^{(2)} = [\tilde{\vx}_2, \vv_3, 0]$ now. Note that the weights and bias in $\pool_{j \in [n]} \phi_{\vngn}^{(2)}$ will be different from those in $\pool_{j \in [n]} \phi_{\vngn}^{(1)}$.

For $i$-th graph node feature, as $\vz_{\vn}^{(1)} = [\tilde{\vx}_1, \vv_2, 0]$ and $\vz^{(1)}_i = [\vx_i, \bm{0}_{d}, 0]$, according to $\gamma_{\gn}^{(k)}$ in  \Cref{eqn-gn-update-function}, $\vz^{(2)}_i =[\vx_i, e^{\widetilde{\alpha'_{i, 1}}}\mW_V \tilde{\vx}_1, e^{\widetilde{\alpha'_{i, 1}}}]$. Here $\widetilde{\alpha'_{i, 1}}:=\alpha'(\vx_i, \tilde{\vx}_1)$. We will use similar notations in later iterations. 
\footnote{To reduce the notation clutter and provide an intuition of the proof, we omit the approximation error introduced by using MLP to approximate
aggregation/message/update function, and assume the aggregation/message/update can be exactly implemented by neural networks. In the proofs, approximation error by MLP is handled rigorously. } 

\textbf{Time $3$}:

Similar to the analysis above, $\vz_{\vn}^{(3)} = [\widetilde{\vx_3}, \vv_4, 0]$.

$\vz_{i}^{(3)} = [\vx_i, e^{\widetilde{\alpha'_{i, 1}}}\mW_V \tilde{\vx}_1 + e^{\widetilde{\alpha'_{i, 2}}}\mW_V \tilde{\vx}_2, e^{\widetilde{\alpha'_{i, 1}}}+e^{\widetilde{\alpha'_{i, 2}}}]$.

\textbf{Time $n$}:

$\vz_{\vn}^{(n)} = [\tilde{\vx}_n, \bm{0}_d, 0]$. 

$\vz_{i}^{(n)} = \vx_i, 
\underbrace{e^{\widetilde{\alpha'_{i, 1}}}\mW_V \tilde{\vx}_1 + ... + e^{\widetilde{\alpha'_{i, n-1}}}\mW_V \widetilde{\vx_{n-1}}}_{n-1 \text{ terms}}, \\
\underbrace{e^{\widetilde{\alpha'_{i, 1}}}+e^{\widetilde{\alpha'_{i, 2}}}+... + e^{\widetilde{\alpha'_{i, n-1}}}]}_{n-1 \text{ terms}}
$.

\textbf{Time $n+1$}:

According to \Cref{subsubsec-vn}, in $n+1$ iteration, the virtual node's feature will be $\bm{1}_{d}$. 

$\vz_{i}^{(n+1)} = [\vx_i, \sum_{k\in [n]} e^{\widetilde{\alpha'_{ik}}}\mW_V\tilde{\vx}_k, \sum_{k\in [n]} e^{\widetilde{\alpha'_{ik}}}]$

\textbf{Time $n+2$ (final layer)}:

For the virtual node, its node feature will stay the same.

For the graph node feature, the last layer will serve as a normalization of the attention score (use MLP to approximate vector-scalar multiplication), and set the last channel to be 0 (projection), resulting in an approximation of $[\vx_i, \frac{\sum_{k\in [n]} e^{\widetilde{\alpha'_{ik}}} \mW_V \tilde{\vx}_k}{\sum_{k\in [n]} e^{\widetilde{\alpha'_{ik}}}}, 0]$. Finally, we need one more linear transformation to make the node feature become $[\frac{\sum_{k\in [n]}  e^{\widetilde{\alpha'_{ik}}} \mW_V \tilde{\vx}_k}{\sum_{k\in [n]} e^{\widetilde{\alpha'_{ik}}}}, \bm{0}_d, 0]$. The first $d$ channel is an approximation of the output of the self-attention layer for node $i$ where the approximation error can be made as small as possible. This is proved in \Cref{sec:approximate-full-self-attention}, and we conclude that heterogeneous MPNN + VN can approximate the self-attention layer $\mL$ to arbitrary precision with $\mc{O}(n)$ MPNN layers.

\subsection{Controlling Error}\label{subsec-controling-error}

On the high level, there are three major sources of approximation error: 1) approximate hard selection with self-attention and 2) approximate equation $\gamma_{\gn}^{(k)}$ with MLPs, and 3) attention normalization in the last layer. 
In all cases, we aim to approximate the output of a continuous map $\mL_c(\vx)$. However, our input is usually not exact $\vx$ but an approximation of $\tilde{\vx}$. We also cannot access the original map $\mL_c$ but instead, an MLP approximation of $\mL_c$, denoted as $\mL_{\MLP}$. The following lemma allows to control the difference between $\mL_c(\vx)$ and $\mL_{\MLP}(\tilde{\vx})$. 

\begin{lemma}
\label{lemma:approximation-meta-lemma}
Let $\mL_c$ be a continuous map from compact set to compact set in Euclidean space. Let $\mL_{\MLP}$ be the approximation of $\mL_c$ by MLP. If we can control $\norm{\vx - \tilde{\vx} }$ to an arbitrarily small degree, we can then control the error $\norm{\mL_c(\vx)-\mL_{\MLP}(\tilde{\vx})}$ arbitrarily small. 
\end{lemma}
\begin{proof}
By triangle inequality $\norm{\mL_c(\vx)-\mL_{\MLP}(\tilde{\vx})} \leq \norm{\mL_c(\vx) - \mL_{\MLP}(\vx))} + \norm{\mL_{\MLP}(\vx) - \mL_{\MLP}(\tilde{\vx})}$. 

For the first term $\norm{\mL_c(\tilde{\vx}) -\mL_{\MLP}(\tilde{\vx})}$, by the universality of MLP, we can control the error $\norm{\mL_c(\tilde{\vx}) -\mL_{\MLP}(\tilde{\vx})}$ in arbitrary degree. 

For the second term $\norm{\mL_{\MLP}(\vx) - \mL_{\MLP}(\tilde{\vx})}$, as $\mL_{\MLP}$ is continuous on a compact domain, it is uniformly continuous by Heine-Cantor theorem. This means that we can control the $\norm{\mL_{\MLP}(\vx) - \mL_{\MLP}(\tilde{\vx})}$ as long as we can control $\norm{\vx - \tilde{\vx}}$, independent from different $\vx$. By assumption, this is indeed the case so we conclude the proof. 
\end{proof}

\begin{remark}
The implication is that when we are trying to approximate the output of a continuous map $\mL_c$ on the compact domain by an MLP $\mL_{\MLP}$, it suffices to show the input is 1) $\norm{\mL_c - \mL_{\MLP}}_{\infty}$ and 2) $\norm{\tilde{\vx}-\vx}$ can be made arbitrarily small. The first point is usually done by the universality of MLP on the compact domain \citep{cybenko1989approximation}. The second point needs to be shown case by case. 

In the \Cref{subsec-a-running-example}, to simplify the notations we omit the error introduced by using MLP to approximate aggregation/message/update functions (continuous functions on the compact domain of $\R{d}$.) in MPNN + VN. \Cref{lemma:approximation-meta-lemma} justify such reasoning.  
\end{remark}

\begin{lemma}[$\tilde{\vx}_i$ approximates $\vx_i$. $\widetilde{\alpha'_{i, j}}$ approximates $\alpha'_{i, j}$.]\label{lemma-approximation-node-feature}
For any $\epsilon>0$ and $x\in \mathcal{X}$, there exist a set of weights for message/aggregate functions of the virtual node such that $||\vx_i - \tilde{\vx}_i||<\epsilon$ and $|\alpha'_{i, j} -\widetilde{\alpha'_{i, j}}| < \epsilon$.
\end{lemma}

\begin{proof}
By \Cref{lemma-uniform-selection} We know that $\widetilde{\alpha_{i, j}} := \widetilde{\alpha}(\vx_i,\vx_j) \rightarrow \delta(i-j)$ as $C_3(\epsilon)$ goes to infinity. Therefore we have
\begin{equation}
||\tilde{\vx}_i - \vx_i|| = ||\sum_j \widetilde{\alpha_{i, j}}\vx_j-\vx_i|| = ||\sum (\widetilde{\alpha}_{i, j} - \delta(i-j))\vx_j|| < \epsilon \sum||\vx_j|| < n C_1 \epsilon
\end{equation}
As $n$ and $C_1$ are fixed, we can make the upper bound as small as we want by increasing $C_3$.

$|\alpha'_{i, j} -\widetilde{\alpha'_{i, j}}| =  |\alpha'(\vx_i, \vx_j) - \alpha'_{\MLP}(\tilde{\vx}_i, \vx_j) | = |\alpha'(\vx_i, \vx_j) - \alpha'(\tilde{\vx}_i, \vx_j) |  + |\alpha'(\tilde{\vx}_i, \vx_j) - \alpha'_{\MLP}(\tilde{\vx}_i, \vx_j)|=   |\alpha'(\vx_i - \tilde{\vx}_i, \vx_j)| = (\vx_i- \tilde{\vx}_i)^T\vx_jC_2^2 + \epsilon < nC_1\epsilon C_1C_2^2 + \epsilon = (nC_1^2C_2^2+1)\epsilon$. As $\alpha'_{i, j}, \widetilde{\alpha'_{i, j}}$ is bounded from above and below, it's easy to see that $|e^{\alpha'_{i, j}} -e^{\widetilde{\alpha'_{i, j}}}| = |e^{\alpha'_{i, j}}(1-e^{\alpha'_{i, j}- \widetilde{\alpha'_{i, j}}})| < C(1-e^{\alpha'_{i, j}- \widetilde{\alpha'_{i, j}}})$ can be controlled to arbitrarily degree. 
\end{proof}

\mainthm*
\begin{proof}
$i$-th MPNN + VN layer will select $\tilde{\vx}_i$, an arbitrary approximation $i$-th node feature $\vx_i$ via attention mechanism. This is detailed in the message/aggregation function of the virtual node in \Cref{subsubsec-vn}. Assuming the regularity condition on feature space $\mc{X}$, detailed in AS\ref{AS-1}, the approximation error can be made as small as needed, as shown in \Cref{lemma-uniform-selection,lemma-approximation-node-feature}. 

Virtual node will then pass the $\tilde{\vx}_i$ to all graph nodes, which computes an approximation of $e^{\alpha'(\tilde{\vx}_i, \vx_j)}, \forall j\in [n]$. This step is detailed in the update function $\gamma_{\gn}^{(k)}$ of graph nodes, which can also be approximated arbitrarily well by MLP, proved in \Cref{lemma-approximate-update-function}. By \Cref{lemma:approximation-meta-lemma}, we have an arbitrary approximation of $e^{\alpha'(\tilde{\vx}_i, \vx_j)}, \forall j\in [n]$, which itself is an arbitrary approximation of $e^{\alpha'(\vx_i, \vx_j)}, \forall j\in [n]$. 

Repeat such procedures $n$ times for all graph nodes, we have an arbitrary approximation of $\sum_{k\in [n]} e^{\alpha'_{ik}}\mW_V \vx_k \in \R{d}$ and $\sum_{k\in [n]} e^{\alpha'_{ik}} \in \R{}$. Finally, we use the last layer to approximate attention normalization $L_c(\vx, y) = \frac{\vx}{y}$, where $\vx \in \R{d}, y\in \R{}$. As inputs for attention normalization are arbitrary approximation of $\sum_{k\in [n]} e^{\alpha'_{ik}}\mW_V \vx_k$ and ${\sum_{k\in [n]} e^{\alpha'_{ik}}}$, both of them are lower/upper bounded according to AS\ref{AS-2} and AS\ref{AS-3}. Since the denominator is upper bounded by a positive number, this implies that the target function $L_c$ is continuous in both arguments. By evoking \Cref{lemma:approximation-meta-lemma} again, we conclude that we can approximate its output $\frac{\sum_{k\in [n]} e^{\alpha'_{ik}}\mW_V \vx_k}{{\sum_{k\in [n]} e^{\alpha'_{ik}}}}$ arbitrarily well. This concludes the proof. 

\end{proof}

\subsection{Relaxing Assumptions with More Powerful Attention}
\label{subsec:relax-assumption}
One limitation of \Cref{thm:constant-width} are assumptions on node features space $\mc{X}$: we need to 1) restrict the variability of node feature so that we can select one node feature to process each iteration. 2) The space of the node feature also need to satisfy certain configuration in order for VN to select it.  For 2), we now consider a different attention function for $\alpha_{\vn}$ in MPNN + VN that can relax the assumptions AS\ref{AS-1} on $\mc{X}$.

\textbf{More powerful attention mechanism.} From proof of \Cref{thm:constant-width}, we just need $\alpha(\cdot, \cdot)$ uniformly select every node in $\mX\in \mathcal{X}$. The unnormalized bilinear attention $\alpha'$ is weak in the sense that $f(\cdot) = \ip{\vx_i\mW_Q\mW_K^T, \cdot}$ has a linear level set. Such a constraint can be relaxed via an improved attention module \gatii. Observing the ranking of the attention scores given by \texttt{GAT} \citep{velivckovic2017graph} is unconditioned on the query node, \citet{brody2021attentive} proposed \gatii, a more expressive attention mechanism. 
In particular, the unnormalized attention score $\alpha'_{\text{GATv2}}(\vu, \vv) :=  \va^T \operatorname{LeakyReLU}\left(\mW \cdot\left[\vu \| \vv\right] + \vb \right)$, where $[\cdot || \cdot ]$ is concatenation. We will let $\alpha_{\vn} = \alpha_{\gat}$ to select features in $\pool_{j \in [n]} \phi_{\vngn}^{(k)}$. 

\begin{figure}[hbtp]
  \centering
  \includegraphics[width=.35\linewidth]{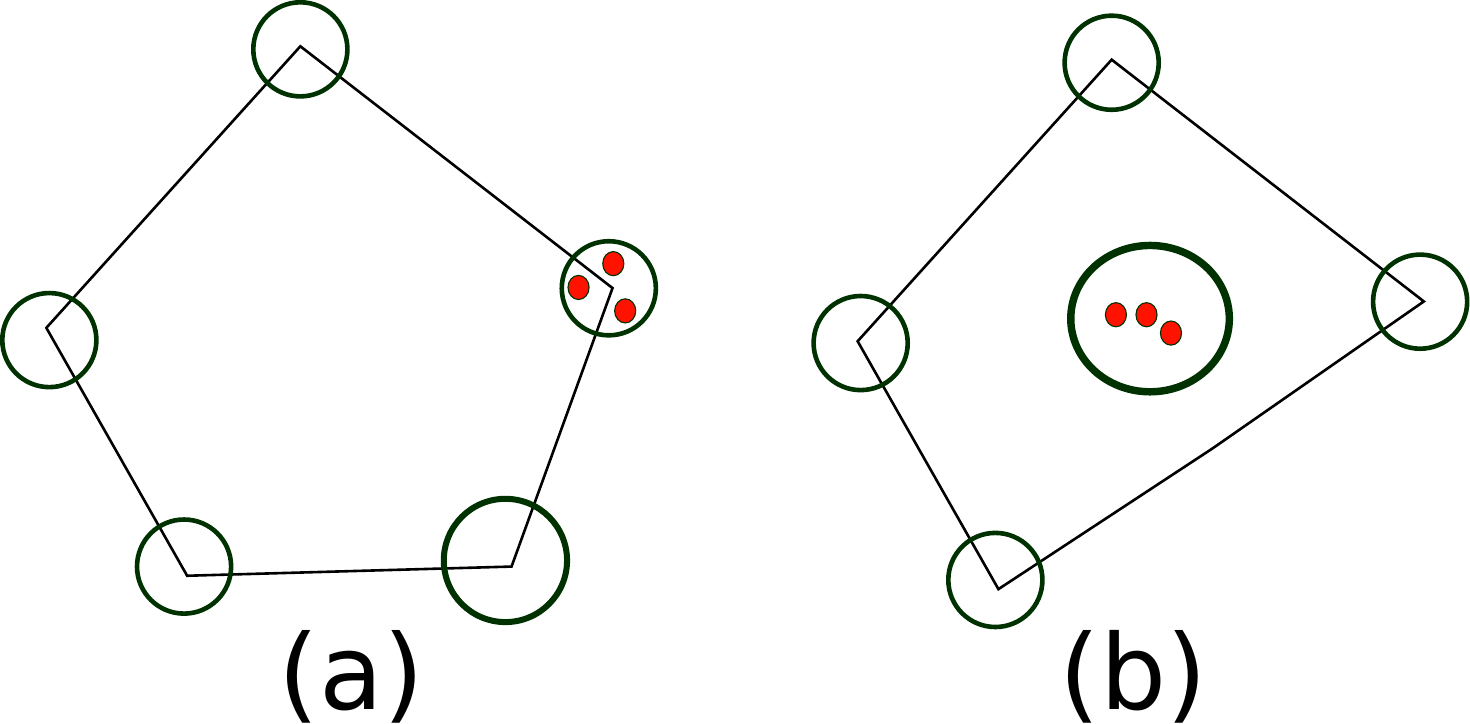}
\caption{In the left figure, we have one example of $\mc{X}$ being $(\mv,\delta)$ separable, for which $\alpha$ can uniformly select any point (marked as red) $\vx_i \in \mc{X}_i$. In the right figure, we change $\alpha_{\vn}$ in MPNN + VN to $\alpha_{\gat}$, which allows us to select more diverse feature configurations. The cluster in the middle cannot be selected by any $\alpha \in \mc{A}$ but can be selected by $\alpha_{\gat}$ according to \Cref{prop:gat-v2-selection}. }
\label{fig:convergence}
\end{figure}

\begin{restatable}{lemma}{gatvtwoUniversality}
\label{lemma-gatv2-universality}
$\alpha'_{\gat}(\cdot, \cdot)$ can approximate any continuous function from $\R{d} \times \R{d} \rightarrow \R{}$. For any $\vv\in \R{d}$, a restriction of $\alpha'_{\gat}(\cdot, \vv)$ can approximate any continuous function from $\R{d} \rightarrow \R{}$.
\end{restatable}
\begin{proof}
Any function continuous in both arguments of $\alpha'_{\gat}$ is also continuous in the concatenation of both arguments. As any continuous functions in $\R{2d}$ can be approximated by  $\alpha'_{\gat}$ on a compact domain according to the universality of MLP \citep{cybenko1989approximation}, we finish the proof for the first statement.

For the second statement, we can write $W$ as $2\times 2$ block matrix and restrict it to cases where only $\mW_{11}$ is non-zero. Then we have 
\begin{equation}
\alpha'_{\gat}(\vu, \vv) 
= a^T \operatorname{LeakyReLU}\left(
  \left[ {\begin{array}{cc}
   \mW_{11} & \mW_{12} \\
   \mW_{21} & \mW_{22} \\
  \end{array} } \right] \cdot \left[\begin{array}{c}
   \vu  \\
   \vv \\
  \end{array} \right]+ \vb\right) 
=\va^T \operatorname{LeakyReLU}\left(\mW_{11}\vu + \vb\right) 
\end{equation} 
  which gives us an MLP on the first argument $\vu$. By the universality of MLP, we conclude the proof for the second statement. 

\end{proof}

\begin{definition}
\label{delta-sepration}
Given $\delta>0$, We call $\mathcal{X}$ is $\delta$ nonlinearly separable if and only if $\min_{i \neq j} d(\mc{X}_i, \mc{X}_j) > \delta$. 
\end{definition}

\begin{customAS}{3'}
\label{AS-4} 
$\mathcal{X}$ is $\delta$ nonlinearly separable for some $\delta>0$. 
\end{customAS}

\begin{restatable}{proposition}{gatvtwoselection}
\label{prop:gat-v2-selection}
If $\mc{X} \subset \R{n \times d}$ satisfies that $\mc{X}_i$ is $\delta$-separated from $\mc{X}_j$ for any $i, j \in [n]$, the following holds. For any $\mX\in \mc{X}$ and $i\in [n]$, there exist a $\alpha_{\gat}$ to select any $\vx_i \in \mc{X}_i$. This implies that we can arbitrarily approximate the self-attention layer $\mL$ after relaxing AS3 to AS3'. 
\end{restatable}
\begin{proof}
For any $i \in [n]$, as $\mc{X}_i$ is $\delta$-separated from other $\mc{X}_j, \forall j\neq i$, we can draw a region $\Omega_i \subset \R{d}$ that contains $\mc{X}_i$ and separate $\mc{X}_i$ from other $\mc{X}_j (j\neq i)$, where the distance from $\mc{X}_i$ from other $\mc{X}_j$ is at least $\delta$ according to the definition of \Cref{delta-sepration}. Next, we show how to construct a continuous function $f$ whose value in $\mc{X}_i$ is at least 1 larger than its values in any other $\mc{X}_j$  $\forall j\neq i$. 


We set the values of $f$ in $\mc{X}_i$ to be 1.5 and values of $f$ in $\mc{X}_j, \forall j\neq i$ to be 0. We can then interpolate $f$ in areas outside of $\cup \mc{X}_i$ (one way is to set the values of $f(x)$ based on $d(x, \mc{X}_i$), which results in a continuous function that satisfies our requirement.   
By the universality of $\alpha_{\gat}$, we can approximate $f$ to arbitrary precision, and this will let us select any $\mc{X}_i$. 
\end{proof}








\section{On the Limitation of MPNN + VN}
\label{sec:limitation-of-approximate-transformer}
Although we showed that in the main paper, MPNN + VN of varying depth/width can approximate the self-attention of full/linear transformers, this does not imply that there is no difference in practice between MPNN + VN and GT. Our theoretical analysis mainly focuses on approximating self-attention without considering computational efficiency. In this section, we mention a few limitations of MPNN + VN compared to GT.  

\subsection{Representation Gap}
The main limitation of deep MPNN + VN approximating full self-attention is that we require a quite strong assumption: we restrict the variability of node features in order to select one node feature to process each iteration. Such assumption is relaxed by employing stronger attention in MPNN + VN but is still quite strong. 

For the large width case, the main limitation is the computational complexity: even though the self-attention layer requires $\mc{O}(n^2)$ complexity, to approximate it in wide MPNN + VN framework, the complexity will become $\mc{O}(n^d)$ where $d$ is the dimension of node features.

We think such limitation shares a similarity with research in universal permutational invariant functions. Both \DS{} \citep{zaheer2017deep} and Relational Network \citep{santoro2017simple} are universal permutational invariant architecture but there is still a representation gap between the two \citep{zweig2022exponential}. Under the restriction to analytic activation functions, one can construct a symmetric function acting on sets of size $n$ with elements in dimension $d$, which can be efficiently
approximated by the Relational Network, but provably requires width exponential in $n$ and $d$ for the \DS{}. We believe a similar representation gap also exists between GT and MPNN + VN and leave the characterization of functions lying in such gap as the future work. 

\subsection{On The Difficulty of Approximating Other Linear Transformers}
\label{subsec:ohter-linear-transformer}
In \Cref{sec:shallow-narrow-attention}, we showed MPNN + VN of $\mc{O}(1)$ width and depth can approximate the self-attention layer of one type of linear transformer, Performer. The literature on efficient transformers is vast \cite{tay2020efficient} and we do not expect MPNN + VN can approximate many other efficient transformers. Here we sketch a few other linear transformers that are hard to  approximate by MPNN + VN of constant depth and width. 

Linformer \citep{wang2020linformer} projects the $n\times d$ dimension keys and values to $k\times d$ suing additional projection layers, which in graph setting is equivalent to graph coarsening. As MPNN + VN still operates on the original graph, it fundamentally lacks the key component to approximate Linformer. 


We consider various types of efficient transformers effectively generalize the virtual node trick. By first switching to a more expansive model and reducing the computational complexity later on, efficient transformers effectively explore a larger model design space than MPNN + VN, which always sticks to the linear complexity.

\subsection{Difficulty of Representing SAN Type Attention}
In SAN \citep{kreuzer2021rethinking}, different attentions are used conditional on whether an edge is presented in the graph or not, detailed below. One may wonder whether we can approximate such a framework in MPNN + VN. 

In our proof of using MPNN + VN to approximate regular GT, we mainly work with \Cref{def:simplified-hetero-mpnn-vn} where we do not use any \gngn{}   edges and therefore not leverage the graph topology. It is straightforward to use \gngn{} edges and obtain the different message/update/aggregate functions for \gngn{} edges non-\gngn{} edges. Although we still achieve the similar goal of SAN to condition on the edge types, it turns out that we can not arbitrarily approximate SAN. 

Without loss of generality, SAN uses two types of attention depending on whether two nodes are connected by the edge. Specifically, 
\begin{equation}
\begin{aligned}
& \hat{\boldsymbol{w}}_{i j}^{k, l}=\left\{\begin{array}{lr}
\frac{\boldsymbol{Q}^{1, k, l} \boldsymbol{h}_i^l \circ \boldsymbol{K}^{1, k, l} \boldsymbol{h}_j^l \circ \boldsymbol{E}^{1, k, l} \boldsymbol{e}_{i j}}{\sqrt{d_k}} & \text { if } i \text { and } j \text { are connected in sparse graph } \\
\frac{\boldsymbol{Q}^{2, k, l} \boldsymbol{h}_i^l \circ \boldsymbol{K}^{2, k, l} \boldsymbol{h}_j^l \circ \boldsymbol{E}^{2, k, l} \boldsymbol{e}_{i j}}{\sqrt{d_k}} & \text { otherwise }
\end{array}\right\} \\
& w_{i j}^{k, l}=\left\{\begin{array}{cc}
\frac{1}{1+\gamma} \cdot \operatorname{softmax}\left(\sum_{d_k} \hat{\boldsymbol{w}}_{i j}^{k, l}\right) & \text { if } i \text { and } j \text { are connected in sparse graph } \\
\frac{\gamma}{1+\gamma} \cdot \operatorname{softmax}\left(\sum_{d_k} \hat{\boldsymbol{w}}_{i j}^{k, l}\right) & \text { otherwise }
\end{array}\right\}
\end{aligned}
\end{equation}
where $\circ$ denotes element-wise multiplication and $\boldsymbol{Q}^{1, k, l}, \boldsymbol{Q}^{2, k, l}, \boldsymbol{K}^{1, k, l}, \boldsymbol{K}^{2, k, l}, \boldsymbol{E}^{1, k, l}, \boldsymbol{E}^{2, k, l} \in$ $\mathbb{R}^{d_k \times d}$. $\gamma \in \mathbb{R}^{+}$is a hyperparameter that tunes the amount of bias towards full-graph attention, allowing flexibility of the model to different datasets and tasks where the necessity to capture long-range dependencies may vary. 

To reduce the notation clutter, we remove the layer index $l$, and edge features, and also consider only one-attention head case (remove attention index $k$). The equation is then simplified to
\begin{equation}
\label{equ:san-simplified-attention}
\begin{aligned}
& \hat{\boldsymbol{w}}_{i j}=\left\{\begin{array}{lr}
\frac{\boldsymbol{Q}^{1} \boldsymbol{h}_i^l \circ \boldsymbol{K}^{1} \boldsymbol{h}_j^l }{\sqrt{d_k}} & \text { if } i \text { and } j \text { are connected in sparse graph } \\
\frac{\boldsymbol{Q}^{2} \boldsymbol{h}_i^l \circ \boldsymbol{K}^{2} \boldsymbol{h}_j^l}{\sqrt{d_k}} & \text { otherwise }
\end{array}\right\} \\
& w_{i j}=\left\{\begin{array}{cc}
\frac{1}{1+\gamma} \cdot \operatorname{softmax}\left(\sum_{d} \hat{\boldsymbol{w}}_{i j}\right) & \text { if } i \text { and } j \text { are connected in sparse graph } \\
\frac{\gamma}{1+\gamma} \cdot \operatorname{softmax}\left(\sum_{d} \hat{\boldsymbol{w}}_{i j}\right) & \text { otherwise }
\end{array}\right\}
\end{aligned}
\end{equation}
We will then show that \Cref{equ:san-simplified-attention} can not be expressed (up to an arbitrary approximation error) in MPNN + VN framework. To simulate SAN type attention, our MPNN + VN framework will have to first simulate one type of attention for all edges, as we did in the main paper, and then simulate the second type of attention between \gngn{} edges by properly offset the contribution from the first attention. This seems impossible (although we do not have rigorous proof) as we cannot express the difference between two attention in the new attention mechanism.

\section{Experimental Details}
\subsection{Dataset Description}
\textbf{ogbg-molhiv} and \textbf{ogbg-molpcba} \citep{hu2020open} are molecular property prediction datasets
adopted by OGB from MoleculeNet. These datasets use a common node (atom) and edge (bond)
featurization that represent chemophysical properties. 
The prediction task of ogbg-molhiv is a binary
classification of molecule's fitness to inhibit HIV replication. The ogbg-molpcba, derived from
PubChem BioAssay, targets to predict the results of 128 bioassays in the multi-task binary classification
setting.

\textbf{ogbg-ppa} \citep{wu2021representing} consists of protein-protein association (PPA) networks derived from
1581 species categorized into 37 taxonomic groups. Nodes represent proteins and edges encode the
normalized level of 7 different associations between two proteins. The task is to classify which of the
37 groups does a PPA network originate from.

\textbf{ogbg-code2} \citep{wu2021representing} consists of abstract syntax trees (ASTs) derived from the source
code of functions written in Python. The task is to predict the first 5 subtokens of the original
function's name. 

\textbf{OGB-LSC PCQM4Mv2} \citep{hu2021ogb} is a large-scale molecular dataset that shares the
same featurization as ogbg-mol* datasets. It consists of 529,434 molecule graphs. The task is to predict the HOMO-LUMO gap, a quantum physical property originally calculated using Density Functional Theory. True labels for original
test-dev and test-challange dataset splits are kept private by the OGB-LSC challenge organizers.
Therefore for the purpose of this paper, we used the original validation set as the test set, while we
left out random 150K molecules for our validation set.

\subsection{Reproducibility}
For LRGB results in \Cref{subsec:lrgb}, we reproduce the original results up to very small differences. 
\begin{table}[h]
    \caption{Reproduce the original results up to small differences. No VN is used. 
    }
    \label{tab:experiments_peptides_reproduce}
    \begin{adjustwidth}{-2.5         cm}{-2.5                    cm}\centering
    \scalebox{0.9}{
    \setlength\tabcolsep{4pt}        
    \begin{tabular}{l                c                           c                  c                             c                               c}\toprule
    \multirow{2}{*}{\textbf{Model}}  &\multirow{2}{*}{\textbf{\# Params.}}          &\multicolumn{2}{c}{\pepfunc} &\multicolumn{2}{c}{\pepstruct} \\                          \cmidrule(lr){3-4} \cmidrule(lr){5-6}
    &                                &\textbf{Test              AP (reproduce)}                &\textbf{Test                 AP                              $\uparrow$}                 &\textbf{Test     MAE (reproduce)}               &\textbf{Test MAE         $\downarrow$} \\\midrule
    GCN                              &508k                       &0.5918$\pm$0.0065 &0.5930$\pm$0.0023            &0.3468$\pm$0.0009              &0.3496$\pm$0.0013          \\
    GINE                             &476k                       &0.5595$\pm$0.0126 &0.5498$\pm$0.0079            &0.3532$\pm$0.0024              &0.3547$\pm$0.0045          \\
    GatedGCN                         &509k                       &0.5886$\pm$0.0027 &0.5864$\pm$0.0077            &0.3409$\pm$0.0011              &0.3420$\pm$0.0013          \\
    GatedGCN+RWSE                    &506k                       &0.6083$\pm$0.0032 &0.6069$\pm$0.0035            &0.3377$\pm$0.0025              &0.3357$\pm$0.0006          \\                 
    \bottomrule
    \end{tabular}
    }
    \end{adjustwidth}
\end{table}

\subsection{The Role of Graph Topology}
In our experiments, we considered graph topology in experiments (i.e., message passing operates on both GN-VN (graph node-virtual node) and GN-GN edges). To understand the role of GN-VN and GN-GN edges, we carried out a set of new experiments where we discard the original graph topology, and only do message passing on GN-VN edges, for Peptides-func \& Peptides-struct datasets. The results are shown in \Cref{tab:graph_topology}.  

We observe that in general, MPNN + VN using GN-VN edges only perform slightly worse than MPNN + VN using both GN-VN and GN-GN edges. However, it still performs better than the standard MPNN without VN. We believe adding VN as a simple way of long-range modeling is the main reason we see good results on Peptides-func \& Peptides-struct datasets. Utilizing local graph topology in MPNN will further improve the performance. 

In general, combining local (message passing) and global modeling (such as GT and VN) in GNN is an active research direction, with novel applications in macromolecule (DNA, RNA, Protein) modeling. In the recent SOTA model GraphGPS \citep{rampavsek2022recipe},  MPNN is interleaved with GT. Consistent with our findings,  \citet{rampavsek2022recipe} also showed both the local component (MPNN) and global component (GT) contribute to the final performance.   

\begin{table}[h]
\label{tab:graph_topology}
\centering
\caption{Utilizing local graph topology in MPNN will further improve the performance on \pepfunc and \pepstruct{}.}
\scalebox{0.65}{
\begin{tabular}{lccc|ccc}
\toprule
& \multicolumn{3}{c}{\pepfunc AP $\uparrow$} & \multicolumn{3}{c}{\pepstruct MAE $\downarrow$} \\
\cline{2-7}
& \textbf{w/o VN (only graph topology)} & \textbf{w/ VN + graph topology} & \textbf{Only VN} & \textbf{w/o VN (only graph topology)} & \textbf{w/ VN + graph topology} & \textbf{Only VN} \\ \midrule
GCN & $0.5930 \pm 0.0023$ & $0.6623 \pm 0.0038$ & $0.6488 \pm 0.0056$ & $0.3496 \pm 0.0013$ & $0.2488 \pm 0.0021$ & $0.2511 \pm 0.0025$ \\
GINE & $0.5498 \pm 0.0079$ & $0.6346 \pm 0.0071$ & $0.6022 \pm 0.0072$ & $0.3547 \pm 0.0045$ & $0.2584 \pm 0.0011$ & $0.2608 \pm 0.0021$ \\
GatedGCN & $0.5864 \pm 0.0077$ & $0.6635 \pm 0.0024$ & $0.6493 \pm 0.0044$ & $0.3420 \pm 0.0013$ & $0.2523 \pm 0.0016$ & $0.2684 \pm 0.0039$ \\
GatedGCN+RWSE & $0.6069 \pm 0.0035$ & $0.6685 \pm 0.0062$ & $0.6432 \pm 0.0072$ & $0.3357 \pm 0.0006$ & $0.2529 \pm 0.0009$ & $0.2645 \pm 0.0023$ \\
\bottomrule
\end{tabular}
}
\end{table}

\subsection{Additional Experiments}
We tested MPNN + VN on \pascal datasets and also observe improvement, shown in \Cref{tab:pascal}, although the improvement is not as large as that of \pepfunc and \pepstruct datasets. The best MPNN + VN model is GatedGCN + LapPE where the performance gap to the best GT model is rather small.  

\begin{table}[hbtp]
    \caption{Baseline experiments for \pascal and \coco with \rbgraph graph on SLIC compactness 30 for the node classification task. The performance metric is macro F1 on the respective splits (Higher is better). All experiments are run 4 times with 4 different seeds. 
    The MP-GNN models are 8 layers deep, while the transformer-based models have 4 layers in order to maintain comparable hidden representation size at the fixed parameter budget of 500k. \textbf{Bold}: Best score.
    }
    \label{tab:pascal}
    \begin{adjustwidth}{-2.5 cm}{-2.5 cm}\centering
    \scalebox{1}{
    \setlength\tabcolsep{4pt} 
    \begin{tabular}{l c c c } 
    \toprule
    \multirow{2}{*}{\textbf{Model}} & \multirow{2}{*}{\hspace*{-2em}\textbf{\# Params}} & \multicolumn{2}{c}{\pascal} \\ \cmidrule(lr){3-4} 
                                    &                                                   & \textbf{Before VN + Test F1}                        & \textbf{After VN + Test F1 $\uparrow$} \\ 
    \midrule 
    GCN                             & 496k                                              & 0.1268$\pm$0.0060                                   & 0.1901$\pm$0.0040                 \\ 
    GINE                            & 505k                                              & 0.1265$\pm$0.0076                                   & 0.1198$\pm$0.0073                 \\ 
    GatedGCN                        & 502k                                              & 0.2873$\pm$0.0219                                   & 0.2874$\pm$0.0178                 \\ 
    GatedGCN+LapPE                  & 502k                                              & 0.2860$\pm$0.0085                                   & \first{0.3103$\pm$0.0068}                 \\ \cmidrule(l){1-4}
    Transformer+LapPE               & 501k                                              & 0.2694$\pm$0.0098                                   & -                 \\ 
    SAN+LapPE                       & 531k                                              & \first{0.3230$\pm$0.0039}                           & -              \\ 
    SAN+RWSE                        & 468k                                              & \second{0.3216$\pm$0.0027}                          & -             \\ 
    \bottomrule
    \end{tabular}
    }
    \end{adjustwidth}
\end{table}


\subsection{Predicting Sea Surface Temperature} \label{appendix:climate}
\label{appendix:climate}
In this experiment, we consider a specific physical modeling problem: forecasting sea surface temperature (SST), that is the water temperature close to the ocean's surface. SST is an essential climate indicator and plays a significant role in analyzing and monitoring the dynamics of weather, climate, and other biological systems for several applications in environmental protection, agriculture, and industry. We use the NOAA/NESDIS/NCEI Daily Optimum Interpolation Sea Surface Temperature (DOISST) version 2.1 proposed by \cite{ImprovementsoftheDailyOptimumInterpolationSeaSurfaceTemperatureDOISSTVersion21} as an improvement upon version 2.0 from \cite{Reynolds:2007}.

We consider the daily SST data of the Pacific Ocean from 1982 to 2021, in the region of longitudes from $180.125^{\circ}\text{E}$ to $269.875^{\circ}\text{E}$ and latitudes from $-14.875^{\circ}\text{N}$ 
to $14.875^{\circ}\text{N}$. We reduce the resolution of the original data from $0.25^{\circ}$-degree 
to $0.5^{\circ}$-degree. 
Following the procedure from \cite{de2018deep}, \cite{deBezenac2019} and \cite{wang2022metalearning}, we divide the region into 11 square batches of equal size (see Table \ref{tab:Pacific-regions}), each contains exactly 30 longitudes and 30 latitudes that can be represented as a grid graph of 900 nodes in which we connect each node to its nearest 8 neighbors. We take time series from 1982 to 2018 as our training set, data in 2019 as our validation set, and data from 2020 to 2021 as our testing set. 
In our experiments, we set the history window $w_h$ as 6 weeks (i.e. 42 days) and the prediction window $w_p$ as 4 weeks (i.e. 28 days), 2 weeks (i.e. 14 days) or 1 week (i.e. 7 days). For each example, each node of the graph is associated with an input time series capturing the temperatures at the corresponding (longitude, latitude) for the last $w_h$ days, and the task is to predict the output time series of temperatures for the next $w_p$ days.

We represent each time series as a long vector and the learning task is fundamentally a node-level regression task. We make sure that there is no overlapping among training, validation and testing sets (e.g., the output of a training example will \textit{not} appear in any input of another validation example). The number of training, validation, and testing examples are roughly 150K, 3K and 7K, respectively for each setting (see Table \ref{tab:noexamples}). We compare our MPNN + VN model with:
\begin{itemize}
\item Multilayer Perceptron (MLP) which treats both the input and output as long vectors and has 512 hidden neurons.
\item TF-Net \cite{Rui2020} with the setting as in the original paper.
\item Linear Transformer \cite{katharopoulos-et-al-2020} \cite{wang2020linformer}\footnote{The Linear Transformer implementation is publicly available at \url{https://github.com/lucidrains/linear-attention-transformer}} with Laplacian positional encoding (LapPE). We compute the first 16 eigenvectors as positions for LapPE.
\end{itemize}
Both MPNN and MPNN + VN have 3 layers of message passing with 256 hidden dimensions. We apply an MLP with one hidden layer of 512 neurons on top of the network to make the final prediction.

We train all our models with 100 epochs with batch size 20, initial learning rate $10^{-3}$, and Adam optimizer \cite{Adam2014}.

\begin{table}
\caption{\label{tab:noexamples} Number of training, validation and testing examples for each setting in the task of SST prediction.}
\begin{center}
\begin{tabular}{ccccc}
\toprule
\textbf{History window} & \textbf{Prediction window} & \textbf{Train size} & \textbf{Validation size} & \textbf{Test size} \\
\midrule
\multirow{2}{*}{6 weeks} & 4 weeks & $147,884$ & $3,245$ & $7,271$ \\
& 2 weeks & $148,038$ & $3,399$ & $7,425$ \\
& 1 week & $148,115$ & $3,476$ & $7,502$ \\
\bottomrule
\end{tabular}
\end{center}
\end{table}

\begin{table}
\caption{\label{tab:Pacific-regions} These are 11 regions of the Pacific in our experiment.}
\begin{center}
\begin{tabular}{ccc}
\toprule
\textbf{Index} & \textbf{Longitudes} & \textbf{Latitues} \\
\midrule
1 & [180.125$^{\circ}\text{E}$, 194.875$^{\circ}\text{E}$] & [-14.875$^{\circ}\text{N}$, -0.125$^{\circ}\text{N}$] \\
2 & [195.125$^{\circ}\text{E}$, 209.875$^{\circ}\text{E}$] & [-14.875$^{\circ}\text{N}$, -0.125$^{\circ}\text{N}$] \\
3 & [210.125$^{\circ}\text{E}$, 224.875$^{\circ}\text{E}$] & [-14.875$^{\circ}\text{N}$, -0.125$^{\circ}\text{N}$] \\
4 & [225.125$^{\circ}\text{E}$, 239.875$^{\circ}\text{E}$] & [-14.875$^{\circ}\text{N}$, -0.125$^{\circ}\text{N}$] \\
5 & [240.125$^{\circ}\text{E}$, 254.875$^{\circ}\text{E}$] & [-14.875$^{\circ}\text{N}$, -0.125$^{\circ}\text{N}$] \\
6 & [255.125$^{\circ}\text{E}$, 269.875$^{\circ}\text{E}$] & [-14.875$^{\circ}\text{N}$, -0.125$^{\circ}\text{N}$] \\
7 & [180.125$^{\circ}\text{E}$, 194.875$^{\circ}\text{E}$] & [0.125$^{\circ}\text{N}$, 14.875$^{\circ}\text{N}$] \\
8 & [195.125$^{\circ}\text{E}$, 209.875$^{\circ}\text{E}$] & [0.125$^{\circ}\text{N}$, 14.875$^{\circ}\text{N}$] \\
9 & [210.125$^{\circ}\text{E}$, 224.875$^{\circ}\text{E}$] & [0.125$^{\circ}\text{N}$, 14.875$^{\circ}\text{N}$] \\
10 & [225.125$^{\circ}\text{E}$, 239.875$^{\circ}\text{E}$] & [0.125$^{\circ}\text{N}$, 14.875$^{\circ}\text{N}$] \\
11 & [240.125$^{\circ}\text{E}$, 254.875$^{\circ}\text{E}$] & [0.125$^{\circ}\text{N}$, 14.875$^{\circ}\text{N}$] \\
\bottomrule
\end{tabular}
\end{center}
\end{table}

\subsection{Connection to Over-Smoothing Phenomenon}
\label{subsec:over-smoothing}
Over-smoothing refers to the phenomenon that deep GNN will produce same features at different nodes after too many convolution layers. Here we draw some connection between VN and common ways of reducing over-smoothing. We think that using VN can potentially help alleviate the over-smoothing problem. In particular, we note that the use of VN can simulate some strategies people use in practice to address over-smoothing. We give two examples below. 

Example 1: In \cite{zhao2019pairnorm}, the two-step method (center \& scale) PairNorm is proposed to reduce the over-smoothing issues. In particular, PairNorm consists of 1) Center and 2) Scale

$$\tilde{\mathbf{x}}_i^c =\tilde{\mathbf{x}}_i-\frac{1}{n} \sum_i \tilde{\mathbf{x}}_i$$

$$\dot{\mathbf{x}}_i = s \cdot \frac{\tilde{\mathbf{x}}_i^c}{\sqrt{\frac{1}{n} \sum_i \left||\tilde{\mathbf{x}}_i^c\right||_2^2}}$$

Where $\tilde{\mathbf{x}}$ is the node features after graph convolution and $s$ is a hyperparameter. The main component for implementing PairNorm is to compute the mean and standard deviation of node features. For the mean of node features, this can be exactly computed in VN. For standard deviation, VN can arbitrarily approximate it using the standard universality result of MLP [5]. If we further assume that the standard deviation is lower bounded by a constant, then MPNN + VN can arbitrarily approximate the PairNorm on the compact set. 

Example 2: In \cite{yang2020revisiting} mean subtraction (same as the first step of PairNorm) is also introduced to reduce over-smoothing. As mean subtraction can be trivially implemented in MPNN + VN, arguments in \citep{yang2020revisiting} (with mean subtraction the revised power Iteration in GCN will lead to the Fiedler vector) can be carried over to MPNN + VN setting. 

In summary, introducing VN allows MPNN to implement key components of \cite{yang2020revisiting,zhao2019pairnorm}, we think this is one reason why we observe encouraging empirical performance gain of MPNN + VN.

\end{document}